\pdfoutput=1
\documentclass[twocolumn]{autart}
\makeatletter
\@ifundefined{c@part}{\newcounter{part}}{}

\makeatother

\newif\ifanonymous
\anonymousfalse   

\usepackage{amsmath, amssymb, mathtools}

\usepackage{booktabs}
\usepackage{enumitem}

\usepackage{graphicx}
\usepackage{tikz}
\usetikzlibrary{shapes.geometric, arrows.meta, positioning, calc}

\usepackage{hyperref}

\usepackage{flushend}

\usepackage{xcolor}

\DeclareMathOperator*{\argmax}{arg\,max}

\renewcommand{\paragraph}[1]{\par\medskip\noindent{\bfseries #1}\enspace}

\renewcommand{\Elproofname}{Proof.}

\begin{document}

\begin{frontmatter}

\title{Runtime-Incremental Transformer for Reinforcement-Learning-Based
       Adaptive Control}

\ifanonymous
  \author{}
  \address{}
\else
  \author[amiens]{Giansalvo Cirrincione}\ead{exin@u-picardie.fr},
  \author[palermo]{Adriano Fagiolini\corauthref{cor1}}\ead{fagiolini@unipa.it}

  \corauth[cor1]{Corresponding author. Tel.\ +39 091 238 63613.
    This paper was not presented at any IFAC meeting.}

  \address[amiens]{Laboratoire LTI, Universit\'e de Picardie
    Jules Verne, Amiens, France}
  \address[palermo]{MIRPALab, Department of Engineering, University
    of Palermo, 90128 Palermo, Italy}
\fi


\begin{abstract}
Learning-based adaptive control of robotic manipulators with
non-observable friction memory has been addressed by attention-based
meta-controllers whose number of attention heads is fixed before
training and is tuned by costly offline search. At long memory
horizons, such fixed-capacity controllers are prone to catastrophic
failures on a sizeable fraction of training seeds. The present
paper introduces a runtime mechanism that grows and prunes the
heads of the attention block during reinforcement learning,
governed by two signals: the effective rank of the on-policy
context distribution, which triggers growth when representational
capacity becomes insufficient, and the per-head output magnitude,
which flags redundant heads for removal. Policy continuity at
growth events and a quantitative bound at prune events are
established analytically. On a two-link manipulator with Stribeck
friction, the proposed mechanism attains full success across all
memory regimes, eliminating the long-horizon failure mode and
removing the need for offline tuning of the head count.
\end{abstract}

\end{frontmatter}


\section{Introduction}
\label{sec:intro}

Robot manipulators are commonly controlled by computed-torque
laws whose gains are adapted online to compensate for
uncertainties in inertia and friction. The passivity-based
construction of Slotine and Li~\cite{SlotineLi1987} handles the
case in which the uncertainty is linear in an unknown parameter
vector; viscous and Coulomb friction fall into this case. Stribeck
and LuGre friction~\cite{Armstrong1994,deWit1995LuGre} do not: the
friction force now depends on a hidden internal state that decays
over a finite time, denoted $\tau_z$ and called the
\emph{memory horizon} of the plant. Since the hidden state cannot
be recovered from the instantaneous joint coordinates, the
closed-loop system is no longer Markovian in the measured state,
and classical adaptive arguments no longer apply directly.

A companion paper by the same
authors~\cite{CirrincioneFagiolini2026TemporalAttention}
addresses the non-Markovian regime by feeding a sliding window of
past measurements into a reinforcement-learning (RL) policy that
outputs the gains of the computed-torque law. The policy is a
single-layer transformer encoder; following the naming in the
companion paper, it is denoted INCRT-1L, a single-layer
instance of the Incremental Transformer (INCRT) family. The number
of attention heads, denoted $K$, is the architectural parameter
that controls the representational capacity of the encoder.
INCRT-1L is trained once per task, with $K$ selected before
training by an incremental rank-tracking
procedure~\cite{CirrincioneINCRT2026} applied to a surrogate
operator. The trained controller closes roughly $35\%$
to $40\%$ of the tracking error left by the computed-torque
baseline at short and medium memory horizons,
$\tau_z \in \{1, 2\}$~s. At the longer horizon $\tau_z = 5$~s, the
same procedure fails on four of ten random seeds; the trained
policies either diverge or collapse onto a payload-invariant
response whose tracking error exceeds that of the
uncompensated baseline. The failure is attributed
in~\cite[Sec.~7.3]{CirrincioneFagiolini2026TemporalAttention} to
local attractors of the RL optimisation that are invisible to the
offline surrogate analysis.

The present paper carries out a programme explicitly announced in
the abstract of~\cite{CirrincioneFagiolini2026TemporalAttention}:
the rank-tracking dynamics are moved inside the RL loop, and the
attention heads are \emph{added and removed during training}
rather than fixed in advance. The resulting architecture is referred to as
\emph{Runtime-INCRT}. Runtime head addition is governed by a
cooldown parameter $\Delta_{\mathrm{grace}}$ (the minimum number
of training steps between two events) and a threshold
$\varepsilon_{\mathrm{grow}}$ on the effective rank of the
recent context-token distribution; runtime head removal is
governed by a threshold $\varepsilon_{\mathrm{prune}}$ on the
per-head output magnitude. The resulting mechanism is simpler in
operation than the Phase-1 / Phase-2 procedure it replaces, yet
yields better tracking in all three memory regimes and removes
the need for per-task hyperparameter search.

Three architectural guarantees are established for the runtime
mechanism (Section~\ref{sec:theory}). Every grow event is
continuous at the policy level, since the newly activated head is
initialised to contribute zero output
(Proposition~\ref{prop:growth-continuity}). Every prune event
changes the policy by at most a prescribed multiple of
$\varepsilon_{\mathrm{prune}}$, and the cumulative change along
the training horizon is bounded in terms of the number of prune
events (Theorem~\ref{thm:policy-continuity}). The active head
count, denoted $K(t)$, converges to a steady-state interval whose
width is set by $\varepsilon_{\mathrm{grow}}$ and
$\varepsilon_{\mathrm{prune}}$, provided the effective rank of
the context-token distribution has a limit
(Proposition~\ref{prop:k-convergence}). Corollary~\ref{cor:P3}
assembles the three results into the no-regression property
required by Problem~\ref{prob:runtime}.

Empirically, Runtime-INCRT is evaluated on the two-degree-of-freedom
(2-DOF) Stribeck benchmark of the companion paper, with ten random
seeds per memory regime $\tau_z \in \{1, 2, 5\}$~s
(Section~\ref{sec:experiments}). A single set of runtime-controller
hyperparameters, shared across the three regimes, gives
$30/30$ successful runs and, at $\tau_z = 5$~s, eliminates the
failure mode reported
in~\cite{CirrincioneFagiolini2026TemporalAttention}. The
controller does not, however, adapt the active head count to
task complexity as one might expect: it saturates the head-count
cap $K_{\max}$ in every run of the main campaign, the growth
schedule is identical across memory regimes, and the prune
trigger never fires. Two ablations clarify what is in fact
happening. Sweeping $K_{\max}$ over $\{2, 4, 8, 16\}$ at
$\tau_z = 2$~s shows that the smallest cap gives the best mean
and the lowest per-seed variance, and $K_{\max} = 16$ is worse
than $K_{\max} = 2$ --- larger capacity does not help. Sweeping
$\Delta_{\mathrm{grace}}$ over $\{1\,000, 2\,000, 5\,000\}$
shows that performance is non-monotone in the cooldown, with the
default value optimal and both extremes producing catastrophic
seed failures. These observations together suggest that
Runtime-INCRT's benefit comes from the training trajectory, not
from the terminal capacity: the gradual activation of heads acts
as a \emph{capacity curriculum} internal to a single RL training
run.

Section~\ref{sec:related} reviews related work.
Sections~\ref{sec:problem}--\ref{sec:arch} fix the problem
formulation and the variable-head attention architecture.
Section~\ref{sec:theory} defines the runtime controller and
proves the three guarantees.
Sections~\ref{sec:benchmark}--\ref{sec:experiments} describe the
experimental protocol and report results.
Section~\ref{sec:discussion} draws the conclusions.
Proofs are in Appendix~\ref{app:proofs}; hyperparameter values in
Appendix~\ref{app:hyperparams}.


\section{Related work}
\label{sec:related}

\paragraph{Adaptive control and memory-aware policies.}
The passivity-based adaptive controller of Slotine and
Li~\cite{SlotineLi1987} and its computed-torque
refinements~\cite{SpongBook} deliver asymptotic tracking when
friction is linear in the unknown parameters. For Stribeck and
LuGre models~\cite{Armstrong1994,deWit1995LuGre}, the friction
force is no longer in the scope of a known regressor, and
classical adaptive arguments no longer apply. A
Lyapunov-shielded residual reinforcement-learning framework for
this regime~\cite{CirrincioneFagiolini2026} adds a torque
correction to the nominal law under closed-form admissibility
projection; the parent paper of the present
work~\cite{CirrincioneFagiolini2026TemporalAttention} lifts this
to parameter-level control, with an online Soft Actor-Critic (SAC)
algorithm~\cite{Haarnoja2018SAC} acting on a windowed history
of motion. Transformer-based RL policies are
otherwise mostly confined to the offline
setting~\cite{Chen2021DecisionTransformer,Janner2021Trajectory};
Long Short-Term Memory~\cite{HochreiterSchmidhuber1997} and
Gated Recurrent Unit~\cite{Cho2014GRU} recurrences are the classical alternatives.
The architecture presented in this paper retains the online-SAC backbone and modifies
only the feature extractor.

\paragraph{Dynamic-capacity architectures and curricula.}
Growing networks have a long history~\cite{Fahlman1990Cascade},
with modern instances in Net2Net~\cite{Chen2016Net2Net} and
progressive training~\cite{Karras2018ProgressiveGAN}; pruning is
the dual operation~\cite{Han2015LearningWeights,Molchanov2017Pruning}.
Differentiable neural architecture search~\cite{Liu2019DARTS} and
its RL counterparts~\cite{Pham2018ENAS} perform the search on a
surrogate and retrain the resulting fixed architecture on the
target task; the Phase-1 / Phase-2 separation of
INCRT~\cite{CirrincioneINCRT2026} and
INCRT-1L~\cite{CirrincioneFagiolini2026TemporalAttention} follow
this pattern. Transformer-specific growth has been explored for
training
acceleration~\cite{Gong2019Stacking,Chen2022bert2BERT}. Curriculum
learning in its task-scheduling form~\cite{Bengio2009Curriculum,Graves2017Automated,Andrychowicz2017HER} is distinct from
capacity curricula: the present paper's finding is that gradual
capacity growth, applied inside a single online RL training run
and driven endogenously by the rank of the on-policy context
distribution, acts as an implicit capacity curriculum. Three
distinctions relative to the above literature are specific to
the present setting: (i) the capacity change takes place during a single
online RL run, not in a search-and-retrain loop; (ii) each change
is policy-preserving by construction
(Proposition~\ref{prop:growth-continuity}), which enables
uninterrupted SAC training across the event; and (iii) the driver
of growth is the effective rank of the context-token distribution
induced by the \emph{current} policy, not of a pre-specified
surrogate.


\section{Problem setup}
\label{sec:problem}

This section fixes the class of systems addressed by the paper
and introduces the notation used in
Sections~\ref{sec:arch}--\ref{sec:experiments}. The formulation is
inherited from the parent
work~\cite{CirrincioneFagiolini2026TemporalAttention}, with the
necessary definitions reproduced here so that the paper is
self-contained.

\subsection{Euler--Lagrange systems with unobservable memory}
\label{sec:problem:dynamics}

Consider rigid-body manipulators with $n$ degrees of freedom,
described by the Euler--Lagrange equation
\begin{equation}
  M(q)\,\ddot q \;+\; C(q, \dot q)\,\dot q \;+\; G(q)
  \;+\; F(q, \dot q, z)
  \;=\; \tau,
  \label{eq:el}
\end{equation}
where $q \in \mathbb{R}^n$ are the generalised coordinates,
$M(q) \succ 0$ is the mass matrix, $C(q, \dot q)\dot q$ collects
Coriolis and centripetal terms, $G(q)$ is the gravity vector, and
$\tau \in \mathbb{R}^n$ is the joint-torque input.
The vector $F(q, \dot q, z)$ accounts for non-conservative effects
--- here, friction --- and depends on an \emph{unobservable memory
state} $z \in \mathbb{R}^{n_z}$ whose evolution is governed by
\begin{equation}
  \dot z \;=\; \zeta(q, \dot q, z).
  \label{eq:z_dyn}
\end{equation}
The state $z$ is not instantaneously recoverable from $(q, \dot q)$
and the flow of~\eqref{eq:z_dyn} is contractive with a finite
time-constant $H_z = \|\partial\zeta/\partial z\|^{-1}$, the
\emph{memory horizon} of the plant. Assumption~\ref{ass:regularity}
collects the regularity conditions on the dynamics inherited
from~\cite{CirrincioneFagiolini2026TemporalAttention}.

\paragraph{Running example: Stribeck friction.}
Throughout the paper, $F$ is instantiated as the Stribeck
friction model
\begin{multline}
  F_s(\dot q, z)
  \;=\; F_c
  \;+\; \bigl(F_s^{\mathrm{max}} - F_c\bigr)
     \exp\!\bigl( -(\dot q / v_s)^2 \bigr) \times \\
     \mathrm{sign}(\dot q)
  \;+\; \sigma\,\dot q
  \;+\; z,
  \label{eq:stribeck}
\end{multline}
with memory dynamics
\begin{equation}
  \dot z \;=\; -\,z/\tau_z \;+\; \lambda_z\,\dot q,
  \label{eq:stribeck_z}
\end{equation}
so that $H_z = \tau_z$. The parameter $\tau_z$ is the principal
sweep axis of the experimental evaluation of
Section~\ref{sec:experiments} and takes values in
$\{1, 2, 5\}$~s. The Stribeck construction is a canonical surrogate
for a broader class of unobservable-memory phenomena in
manipulation, including joint elasticity, soft-contact hysteresis,
and damper dynamics~\cite{Armstrong1994,deWit1995LuGre}.

\subsection{Tracking task and extended state}
\label{sec:problem:task}

Given a smooth reference trajectory $q_d(t)$, the control task is
to track $q_d$ under unknown memory $z(t)$ and an unknown static
payload parameter $p \in [p_{\min}, p_{\max}]$ that modifies the
effective mass in $M$. The tracking error is $e = q_d - q$ with
velocity error $\dot e = \dot q_d - \dot q$. Following the
conventions of~\cite{CirrincioneFagiolini2026TemporalAttention},
the extended state used throughout the paper is defined as
\begin{equation}
  x \;=\; (q,\, \dot q,\, e,\, \dot e,\, s)
  \;\in\; \mathbb{R}^{5n},
  \label{eq:extended_state}
\end{equation}
where $s$ collects auxiliary variables (sliding surface, filtered
references). The cost functional is
\begin{equation}
  \mathcal{J}(\pi) \;=\;
    \mathbb{E}_{p,\, z(\cdot),\, q_d(\cdot)}
    \Biggl[\int_0^T \ell\bigl(e(t), \dot e(t)\bigr)\,dt\Biggr],
  \label{eq:cost}
\end{equation}
with $\ell$ strongly convex in its arguments. The expectation is
taken over the task distribution (random payload, reference, and
initial condition) and over the realisation of the memory state
$z(t)$ conditional on the motion.

\subsection{Controller template and admissibility}
\label{sec:problem:controller}

The torque applied to the manipulator follows the parameterised
computed-torque structure of~\cite{CirrincioneFagiolini2026TemporalAttention}
\begin{equation}
  \tau(t) \;=\;
    \mathrm{CT}\bigl(q, \dot q, q_d;\, K_d(t), \Lambda(t)\bigr)
  \;+\;
    \phi_{\mathrm{ff}}\bigl(q, \dot q;\, \eta(t)\bigr),
  \label{eq:tau_template}
\end{equation}
where $\mathrm{CT}$ is the computed-torque law with gain matrices
$K_d, \Lambda$, and $\phi_{\mathrm{ff}}(q, \dot q; \eta)
= \Phi(q, \dot q)\,\eta$ is affine in the feed-forward weights
$\eta$. The parameter trajectory
$\theta_{\mathrm{ctrl}}(t) = (K_d(t), \Lambda(t), \eta(t))$ lives
in a compact box
$\mathcal{P} =
 [K_d^{\min}, K_d^{\max}] \times
 [\Lambda^{\min}, \Lambda^{\max}] \times
 [\eta^{\min}, \eta^{\max}]^{\dim \eta}$.
Affinity of $\phi_{\mathrm{ff}}$ in $\eta$ is retained to preserve
the convexity of the admissibility set of~\cite[Lemma~4]{CirrincioneFagiolini2026TemporalAttention};
the state-dependent feature matrix $\Phi$ may be the output of an
arbitrary deep network without compromising this convexity.

\subsection{Standing assumptions}
\label{sec:problem:assumptions}

The theoretical results of Section~\ref{sec:theory} and the
experimental protocol of Section~\ref{sec:experiments} rely on the
following standing assumptions, inherited
from~\cite{CirrincioneFagiolini2026TemporalAttention} and standard
in the computed-torque
literature~\cite{SlotineLi1987,SpongBook}.

\begin{assum}[Regularity of the dynamics]
\label{ass:regularity}
$M, C, G$ are smooth on the operating set; $F$ is continuous in
$(q, \dot q, z)$ and Lipschitz in $z$; $\zeta$ is Lipschitz and
generates a contractive flow with rate $1/H_z$.
\end{assum}

\begin{assum}[Reference regularity]
\label{ass:reference}
$q_d \in C^2([0, T])$ with uniformly bounded derivatives.
\end{assum}

\begin{assum}[Parameter compactness]
\label{ass:compact}
The parameter set $\mathcal{P}$ is a compact box; the squashing
functions of Section~\ref{sec:arch:output} land in its interior by
construction.
\end{assum}

\subsection{The meta-controller family}
\label{sec:problem:meta}

The meta-controller is a map
\begin{equation}
  G_\theta \,:\, \mathbb{R}^{W \times d_c}
    \;\longrightarrow\; \mathcal{P},
  \label{eq:meta_map}
\end{equation}
parameterised by $\theta \in \mathbb{R}^{d_\theta}$, that consumes
a \emph{context token}
\begin{multline}
  c(t_k)
  \;=\; \bigl(\, o(t_{k-W+1}),\, o(t_{k-W+2}),\, \ldots,\, o(t_k)\,\bigr) \\
  \;\in\; \mathbb{R}^{W \times d_c},
  \label{eq:context}
\end{multline}
and produces the gain tuple
$\theta_{\mathrm{ctrl}}(t) = G_\theta(c(t))$
to be substituted into~\eqref{eq:tau_template}.
Each step-observation $o(t)$ concatenates
$(q,\dot q,\, q_d, \dot q_d,\, \hat p,\, \hat\mu,\, t/T) \in \mathbb{R}^{d_c}$, with $\hat p$ a noisy payload estimate,
$\hat\mu$ a friction-regime estimate held at $0.2$, and $t/T$ a
phase indicator. For the experiments of
Section~\ref{sec:experiments} the degree of freedom is $n=2$, so
$d_c = 4n + 3 = 11$.

\paragraph{Window selection.}
The window size $W$ is the unique hyperparameter of the context
token not fixed by the dynamics. The lower bound of
Proposition~1(iii) of~\cite{CirrincioneFagiolini2026TemporalAttention}
requires $W \geq H_z / \Delta t$ to close the Markovian optimality
gap, which for Stribeck gives $W \geq \tau_z/\Delta t$: at
$\Delta t = 10$~ms, this reads $W \geq 100$ for $\tau_z = 1$~s,
$W \geq 200$ for $\tau_z = 2$~s, and $W \geq 500$ for $\tau_z = 5$~s.
In the main experiments $W = 20$ is adopted at $\tau_z \in \{1, 5\}$ and
$W = 50$ at $\tau_z = 2$; both values are below the Markov-optimality
bound at the corresponding horizon, which is a deliberate stress
test of the windowed architecture and matches the setting
of~\cite{CirrincioneFagiolini2026TemporalAttention}.

\subsection{Problem statement}
\label{sec:problem:statement}

The meta-controller family is indexed by the architectural
parameter $K$, the number of attention heads consumed by
$G_\theta$ to process the history window. The parent
work~\cite{CirrincioneFagiolini2026TemporalAttention} selects $K$
offline, by applying an incremental rank-tracking
analysis~\cite{CirrincioneINCRT2026} to a surrogate temporal
residual operator, before training the policy via reinforcement
learning at fixed $K = K^\star$. This \emph{Phase-1 / Phase-2
separation}, while principled and empirically effective at
$\tau_z \in \{1, 2\}$~s, admits a documented failure mode at the
long-memory regime $\tau_z = 5$~s, where $4/10$ training seeds
diverge or collapse to a payload-invariant policy
(\cite[Section~7.3]{CirrincioneFagiolini2026TemporalAttention}).
The failure is attributed in the parent work to ``local attractors
of the reinforcement-learning optimisation'' that are not
detectable from the surrogate analysis.

This paper addresses the following problem.

\begin{prob}[Runtime head selection]
\label{prob:runtime}
Design a procedure that selects the active head count $K(t)$
\emph{during} reinforcement-learning training, such that
\begin{enumerate}[label=(P\arabic*)]
  \item \label{P1}
    the per-task offline architecture search of Phase~1 is
    obviated;
  \item \label{P2}
    the resulting policy at training termination matches or
    exceeds, in tracking performance, the best fixed-$K$ policy
    selected by the Phase-1 / Phase-2 procedure, across the memory
    horizons $\tau_z \in \{1, 2, 5\}$~s;
  \item \label{P3}
    the transformation of $K(t)$ during training satisfies a
    no-regression property: the policy output is continuous at
    every head-count change, and the parameter energy is bounded
    along the event sequence.
\end{enumerate}
\end{prob}

Requirement~\ref{P1} is a simplification of the engineering
pipeline; requirement~\ref{P2} is the empirical criterion for
success; requirement~\ref{P3} is the architectural guarantee that
distinguishes a principled runtime adaptation from a naive
restart-and-resize.


\section{Variable-head attention architecture}
\label{sec:arch}

The meta-controller architecture used to address
Problem~\ref{prob:runtime} is specified next. The key design is a
\emph{variable-head attention} block whose active head count $K$
can be modified during training via two primitive operations ---
\emph{grow} and \emph{prune} --- without interrupting the
reinforcement-learning loop. Section~\ref{sec:theory} gives the controller that schedules
these operations and establishes the guarantees demanded
by~\ref{P3}.

\subsection{Variable-head attention block}
\label{sec:arch:block}

Let $K_{\max}$ denote the head-count cap fixed at design time, and
$d_k$ the per-head dimension; the model dimension is
$d_{\mathrm{model}} = K_{\max} \cdot d_k$. The block consists of
$K_{\max}$ heads, each parameterised by four matrices
$(W^Q_i,\, W^K_i,\, W^V_i,\, W^O_i) \in
 \mathbb{R}^{d_{\mathrm{model}} \times d_k}
 \times
 \mathbb{R}^{d_{\mathrm{model}} \times d_k}
 \times
 \mathbb{R}^{d_{\mathrm{model}} \times d_k}
 \times
 \mathbb{R}^{d_k \times d_{\mathrm{model}}}$,
and an \emph{active mask} $m \in \{0,1\}^{K_{\max}}$. The forward
pass on a sequence $h \in \mathbb{R}^{W \times d_{\mathrm{model}}}$
is
\begin{equation}
  \mathrm{VarAttn}(h)
  \;=\; h
  \;+\; \sum_{i \,:\, m_i = 1} \mathrm{head}_i(h),
  \label{eq:varattn}
\end{equation}
where the causal per-head map is
\begin{equation}
  \mathrm{head}_i(h)
  \;=\;
  \mathrm{softmax}\!\Bigl(
    \frac{h W^Q_i\, (h W^K_i)^\top}{\sqrt{d_k}}
    \;+\; \mathcal{M}
  \Bigr)\,
  h W^V_i \, W^O_i,
  \label{eq:head}
\end{equation}
and $\mathcal{M}$ is the upper-triangular causal mask. The residual
connection in~\eqref{eq:varattn} is a deliberate choice: it ensures
that a head whose output is zero contributes zero to the forward
pass, a property exploited by the grow operation below.
The meta-controller $G_\theta$ is realised by stacking the
variable-head attention block on the input-embedded context token,
taking the output of the last time-step as the feature vector, and
passing it through a fixed MLP head
$\mathbb{R}^{d_{\mathrm{model}}} \to \mathcal{P}$
(Section~\ref{sec:arch:output}).

The \emph{active head count} at step $t$ is
\begin{equation}
  K(t) \;=\; \textstyle\sum_{i=1}^{K_{\max}} m_i(t).
\end{equation}
The two primitive operations that modify $K(t)$ are described next.

\subsection{Grow and prune primitives}
\label{sec:arch:primitives}

\paragraph{Grow.}
The $\mathrm{Grow}(i)$ operation activates a previously-inactive head
$i$ (i.e., $m_i = 0 \to m_i = 1$) and reinitialises its parameters
according to the following scheme, denoted \emph{soft
initialisation}:
\begin{align}
  W^Q_i, \, W^K_i, \, W^V_i
    &\;\leftarrow\; \text{Xavier normal}, \label{eq:soft_init_qkv} \\
  W^O_i
    &\;\leftarrow\; 0. \label{eq:soft_init_o}
\end{align}
The rationale for~\eqref{eq:soft_init_o} is that
$\mathrm{head}_i(h) = (\cdots) W^O_i$, so setting $W^O_i = 0$ makes
$\mathrm{head}_i$ contribute zero to the sum
in~\eqref{eq:varattn} at the instant of growth, regardless of the
query/key/value initialisation. The resulting policy output is
therefore \emph{unchanged} at the growth event, a property
stated formally and numerically verified in
Proposition~\ref{prop:growth-continuity}. After growth, all four
matrices of head $i$ are registered as trainable in the
reinforcement-learning optimiser.

\paragraph{Prune.}
The $\mathrm{Prune}(i)$ operation deactivates an active head $i$
($m_i = 1 \to m_i = 0$). The parameters of head $i$ are frozen
(\texttt{requires\_grad = False}) and removed from the optimiser
state, so that no further gradient update affects them. The removed
head is excluded from the sum in~\eqref{eq:varattn}, which induces
a one-step discontinuity in the policy output of magnitude at most
$\|\mathrm{head}_i(h)\|_2$ at the pruning instant. This
discontinuity is bounded by the prune-trigger threshold by design
(Section~\ref{sec:theory}).

\paragraph{Bookkeeping.}
Grow and prune events invalidate the optimiser's internal state
(Adam's first- and second-moment buffers), since the set of
trainable parameters changes. The optimiser is rebuilt after every
event by re-registering the active parameters and carrying the
current learning rate forward from the scheduler. The replay buffer
of the reinforcement-learning algorithm is \emph{not} cleared:
Proposition~\ref{prop:growth-continuity} ensures that the stored
transitions remain on-policy at the growth event up to numerical
precision; at prune events the transitions are at most
$\varepsilon_{\mathrm{prune}}$-off-policy, where
$\varepsilon_{\mathrm{prune}}$ is the pruning threshold
(Section~\ref{sec:theory}).

\subsection{Output mapping to $\mathcal{P}$}
\label{sec:arch:output}

The MLP head of $G_\theta$ produces a raw action
$a \in \mathbb{R}^{\dim \mathcal{P}}$, which is squashed into the
compact parameter set $\mathcal{P}$ by
\begin{align}
  K_d(t) &\;=\; K_d^{\min}
    + \sigma(a_K) \odot (K_d^{\max} - K_d^{\min}), \\
  \Lambda(t) &\;=\; \Lambda^{\min}
    + \sigma(a_\Lambda) \odot (\Lambda^{\max} - \Lambda^{\min}), \\
  \eta(t) &\;=\; \eta^{\max} \odot \tanh(a_\eta),
\end{align}
where $\sigma$ denotes the sigmoid. The squashing is identical
to~\cite[Section~4.3]{CirrincioneFagiolini2026TemporalAttention}
and ensures by construction that
$\theta_{\mathrm{ctrl}}(t) \in \mathcal{P}$, preserving
Assumption~\ref{ass:compact} throughout training.

\subsection{Training loop}
\label{sec:arch:training}

The meta-controller is trained by SAC on the single-step meta
Markov decision process (MDP)
of~\cite[Section~4.4]{CirrincioneFagiolini2026TemporalAttention}.
The shielded admissibility constraint of the parent work is
retained unchanged: at every SAC update, the loss is augmented by
a Lagrangian penalty on the distance of $G_\theta(c(t))$ from the
admissible set $\Pi_{\mathrm{adm}}(x(t))$, and a final projection
onto $\Pi_{\mathrm{adm}}(x(t))$ is applied at runtime
(\cite[Theorem~3(b)]{CirrincioneFagiolini2026TemporalAttention}).
The new element in this paper is the runtime controller
$\mathcal{C}_{\mathrm{rt}}$ that invokes
$\mathrm{Grow}$ and $\mathrm{Prune}$ on the variable-head attention
block at selected training steps, interleaved with the SAC updates.
The controller is defined in Section~\ref{sec:theory}; its
interface is specified here.

\begin{description}
  \item[Inputs.] At every \emph{check step}
    $t \in \{t_{\mathrm{chk}},\, 2 t_{\mathrm{chk}}, \ldots\}$,
    $\mathcal{C}_{\mathrm{rt}}$ reads the current active head
    count $K(t)$, a rolling buffer of recent encoder-input tokens,
    and the per-head output norms evaluated on a fresh mini-batch.
  \item[Outputs.] One of three actions: \emph{noop}, $\mathrm{Grow}(i)$
    for some inactive $i$, or $\mathrm{Prune}(i)$ for some active
    $i \neq \argmax$.
  \item[Frequency.] Check steps occur at a fixed cadence
    $t_{\mathrm{chk}}$; consecutive events are separated by a
    minimum cooldown $\Delta_{\mathrm{grace}} \geq t_{\mathrm{chk}}$
    to allow the policy to adapt to the architectural change before
    the next one is triggered.
\end{description}


\section{Runtime controller and guarantees}
\label{sec:theory}

This section specifies the runtime controller
$\mathcal{C}_{\mathrm{rt}}$ that schedules the grow and prune
operations of Section~\ref{sec:arch:primitives} during
reinforcement-learning training, and establishes three results that
jointly answer requirement~\ref{P3} of Problem~\ref{prob:runtime}.
Section~\ref{sec:theory:rank} defines the online effective-rank
signal on which the grow trigger is based.
Section~\ref{sec:theory:rules} gives the deterministic event rules.
Section~\ref{sec:theory:guarantees} states and proves the three
guarantees: Proposition~\ref{prop:growth-continuity}
(zero-loss-at-growth), Theorem~\ref{thm:policy-continuity} (a
policy-continuity bound along the entire event sequence), and
Proposition~\ref{prop:k-convergence} (convergence of the active
head count to a steady-state interval). Corollary~\ref{cor:P3}
summarises how the three results jointly deliver~\ref{P3}.
Proofs are given inline; a fully rigorous treatment of
Proposition~\ref{prop:k-convergence} is deferred to
Appendix~\ref{app:proofs}.

\subsection{Online effective-rank signal}
\label{sec:theory:rank}

The grow rule is driven by an estimate of the effective rank of
the stream of encoder-input tokens, i.e., of the distribution of
context tokens $c(t)$ observed during training. The estimate is
\emph{online} and \emph{non-anticipative}: it uses only tokens
collected up to the current training step.

\begin{defn}[Empirical effective rank]
\label{def:rank}
Fix a buffer size $N$ and an energy threshold
$\alpha \in (0, 1)$. At training step $t$, let
$B_t = [\, c(t - N + 1), \ldots, c(t)\,]^\top \in \mathbb{R}^{N \times W d_c}$
be the matrix of the last $N$ context tokens (flattened), centred
column-wise. Let $\sigma_1(t) \geq \sigma_2(t) \geq \cdots \geq 0$
be the singular values of $B_t$. The \emph{empirical effective
rank} at threshold $\alpha$ is
\begin{equation}
  \hat\rho_\alpha(t)
  \;:=\;
  \min \Biggl\{\, k \geq 1 \;:\;
    \frac{\sum_{j=1}^{k} \sigma_j(t)}{\sum_{j \geq 1} \sigma_j(t)}
    \;\geq\; \alpha \,\Biggr\}.
  \label{eq:rank}
\end{equation}
\end{defn}

Throughout the paper, $N = 10^3$ and $\alpha = 0.95$ are used, as in
the main experiments of Section~\ref{sec:experiments}. The
quantity $\hat\rho_\alpha(t)$ takes integer values in
$[1, W d_c]$ and is bounded independently of $K(t)$.

\begin{rem}[Relation to the INCRT rank signal]
\label{rem:rank_vs_incrt}
Definition~\ref{def:rank} is the on-line analogue of the residual
rank used in the Phase-1 search of
\cite[Section~6]{CirrincioneFagiolini2026TemporalAttention}.
The parent procedure evaluates $r_{\mathrm{eff}}$ once, on a
surrogate temporal residual operator constructed offline, prior to
policy training. The present estimate is updated at every check
step and does not require the construction of a surrogate
operator: it measures directly the rank of the distribution of
context tokens induced by the \emph{current} policy.
\end{rem}

\subsection{Event rules}
\label{sec:theory:rules}

Let $t_{\mathrm{chk}}$ be the check cadence and
$\Delta_{\mathrm{grace}} \geq t_{\mathrm{chk}}$ the cooldown.
At each check step $t = n\, t_{\mathrm{chk}}$, the controller
evaluates the grow trigger, then the prune trigger, and applies
at most one event per check. Figure~\ref{fig:controller-logic}
summarises the decision flow.

\begin{figure}[t]
\centering
\resizebox{\columnwidth}{!}{%
\begin{tikzpicture}[
  font=\scriptsize,
  every node/.style={align=center},
  block/.style={rectangle, rounded corners=2pt, draw, thin,
    minimum width=26mm, minimum height=5.5mm, inner sep=2pt},
  dec/.style={rectangle, rounded corners=2pt, draw, thin,
    minimum width=34mm, minimum height=6mm, inner sep=2pt},
  act/.style={rectangle, rounded corners=2pt, draw, thin,
    fill=gray!12, minimum width=11mm, minimum height=5.5mm, inner sep=2pt},
  noop/.style={rectangle, rounded corners=2pt, draw, thin,
    fill=gray!12, minimum width=10mm, minimum height=5.5mm, inner sep=2pt},
  arr/.style={-{Latex[length=1.1mm]}, thin},
  lbl/.style={font=\tiny, inner sep=0.5pt}
]

\node[block]                     (chk)  {Check step $t_{\mathrm{chk}}$};
\node[dec,   below=4mm of chk]   (cool) {Cooldown $\geq\Delta_{\mathrm{grace}}$?};
\node[block, below=4mm of cool]  (sig)  {Measure $\hat\rho_\alpha(t)$, $\pi_i(t)$};
\node[dec,   below=4mm of sig]   (gc)   {$\hat\rho_\alpha{>}K(t)(1{+}\varepsilon_{\mathrm{grow}})$?};
\node[dec,   below=4mm of gc]    (pc)   {$\min_i\pi_i{<}\varepsilon_{\mathrm{prune}}$?};
\node[block, below=4mm of pc]    (sac)  {Resume SAC update};

\node[act, left=12mm of gc] (grow)  {Grow};
\node[act, left=12mm of pc] (prune) {Prune};

\node[noop, right=12mm of cool] (noop1) {Noop};
\node[noop, right=12mm of pc]   (noop2) {Noop};

\draw[arr] (chk)  -- (cool);
\draw[arr] (cool) -- node[lbl,right] {yes} (sig);
\draw[arr] (sig)  -- (gc);
\draw[arr] (gc)   -- node[lbl,right] {no}  (pc);
\draw[arr] (pc)   -- node[lbl,right] {no}  (sac);

\draw[arr] (cool.east) -- node[lbl,above] {no}  (noop1.west);
\draw[arr] (gc.west)   -- node[lbl,above] {yes} (grow.east);
\draw[arr] (pc.west)   -- node[lbl,above] {yes} (prune.east);
\draw[arr] (pc.east)   --                        (noop2.west);

\draw[arr] (grow.south)  -- ++(0,-2mm)
                         -| ([xshift=-3mm]grow.west)
                         -- ([xshift=-3mm]prune.west |- sac.west)
                         -- (sac.west);
\draw[arr] (prune.south) -- ++(0,-2mm)
                         -| ([xshift=-1mm]prune.west)
                         -- ([xshift=-1mm]prune.west |- sac.west)
                         -- (sac.west);

\coordinate (rail) at ([xshift=5mm]noop1.east);
\draw (noop1.east)  -- (rail);
\draw (noop2.east)  -- ([xshift=5mm]noop2.east);
\draw[arr] (rail)
        -- ([xshift=5mm]noop2.east |- sac.east)
        -- (sac.east);

\end{tikzpicture}}%
\caption{Decision logic of the runtime controller
$\mathcal{C}_{\mathrm{rt}}$ at each check step. A grow event is
triggered when the empirical effective rank $\hat\rho_\alpha$
exceeds the current head count scaled by
$(1+\varepsilon_{\mathrm{grow}})$; a prune event is triggered
when the relative output $\pi_i$ of the weakest active head
falls below $\varepsilon_{\mathrm{prune}}$. At most one event
fires per check; consecutive events are separated by a cooldown
of $\Delta_{\mathrm{grace}}$ steps.}
\label{fig:controller-logic}
\end{figure}
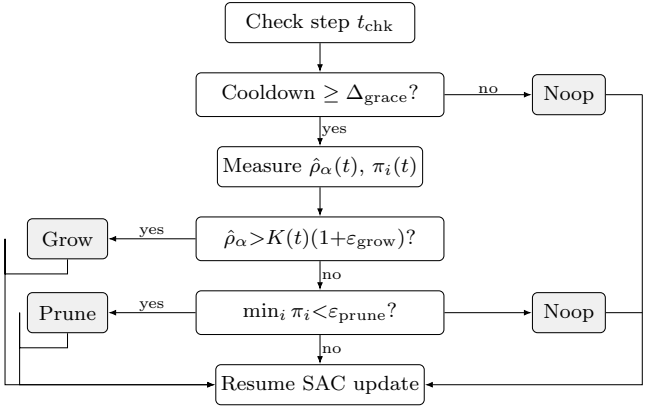

\paragraph{Grow trigger.}
Let $K(t)$ denote the current active head count.
The controller maintains a counter $\kappa_{\mathrm{grow}}(t)$
incremented whenever
\begin{equation}
  \hat\rho_\alpha(t)
  \;>\;
  K(t)\, \bigl( 1 + \varepsilon_{\mathrm{grow}} \bigr)
  \qquad \text{and} \qquad
  K(t) < K_{\max},
  \label{eq:grow_condition}
\end{equation}
and reset to zero otherwise. When
$\kappa_{\mathrm{grow}}(t) \geq \delta_{\mathrm{grow}}$ and the
cooldown has elapsed (at least $\Delta_{\mathrm{grace}}$ training
steps since the previous event), the controller selects an
inactive head index $i$ and performs $\mathrm{Grow}(i)$. The grow
counter is reset and the cooldown clock restarted.

\paragraph{Prune trigger.}
For each active head $i$, let
$\eta_i(t) := \| \mathrm{head}_i(h_t) \|_2$ be the output norm of
head~$i$ on a fresh mini-batch of encoder inputs $h_t$, and let
\begin{equation}
  \pi_i(t)
  \;:=\;
  \frac{\eta_i(t)}
       {\sum_{j :\, m_j(t) = 1} \eta_j(t)}
  \label{eq:prune_fraction}
\end{equation}
be its relative contribution to the attention block's output.
The controller maintains a counter
$\kappa_{\mathrm{prune}}^{(i)}(t)$ incremented whenever
$\pi_i(t) < \varepsilon_{\mathrm{prune}}$ and $K(t) > K_{\min}$,
and reset to zero otherwise. When
$\kappa_{\mathrm{prune}}^{(i)}(t) \geq \delta_{\mathrm{prune}}$
and the cooldown has elapsed, the controller performs
$\mathrm{Prune}(i)$ and restarts the cooldown.

\paragraph{Hyperparameter set.}
The runtime controller exposes six hyperparameters:
$\varepsilon_{\mathrm{grow}}, \varepsilon_{\mathrm{prune}},
 \delta_{\mathrm{grow}}, \delta_{\mathrm{prune}},
 \Delta_{\mathrm{grace}}, t_{\mathrm{chk}}$. Their values for
the main campaign are summarised in Table~\ref{tab:runtime-hp}
of Appendix~\ref{app:hyperparams}; empirical sensitivity to each
is reported in Section~\ref{sec:experiments}.

\subsection{Guarantees}
\label{sec:theory:guarantees}

The three properties announced above are established below. The
first is purely architectural and holds exactly by construction.

\begin{prop}[Zero-loss-at-growth]
\label{prop:growth-continuity}
Let $t^\ast$ be the training step at which $\mathrm{Grow}(i)$ is
applied, with the soft initialisation~\eqref{eq:soft_init_qkv}--\eqref{eq:soft_init_o}.
Then for every input sequence $h$ to the variable-head attention
block,
\begin{equation}
  \mathrm{VarAttn}_{m(t^\ast+)}(h)
  \;=\;
  \mathrm{VarAttn}_{m(t^\ast-)}(h),
  \label{eq:growth_continuity}
\end{equation}
and consequently, on every context token $c$,
\begin{equation}
  G_{\theta(t^\ast+)}(c) \;=\; G_{\theta(t^\ast-)}(c).
  \label{eq:policy_continuity_grow}
\end{equation}
\end{prop}

\begin{pf}
By~\eqref{eq:varattn},
\begin{align*}
  \mathrm{VarAttn}_{m(t^\ast+)}(h)
  &= h + \sum_{j \,:\, m_j(t^\ast+) = 1} \mathrm{head}_j(h) \\
  &= \mathrm{VarAttn}_{m(t^\ast-)}(h) + \mathrm{head}_i(h),
\end{align*}
since $\mathrm{Grow}(i)$ adds exactly one term to the active-head
sum. From~\eqref{eq:head},
$\mathrm{head}_i(h) = \mathrm{softmax}(\cdots)\, h W^V_i\, W^O_i$.
The soft initialisation~\eqref{eq:soft_init_o} sets
$W^O_i = 0$, so $\mathrm{head}_i(h) = 0$ and the two attention
outputs coincide. Since all downstream blocks of $G_\theta$
(layer-normalisation, MLP head, squashing) are deterministic
in their inputs and are not affected by the mask update,
$G_{\theta(t^\ast-)}(c) = G_{\theta(t^\ast+)}(c)$ for all $c$.
\qed
\end{pf}

\begin{rem}[Numerical verification]
\label{rem:numeric-growth}
In the PyTorch implementation used for the experiments of
Section~\ref{sec:experiments}, the two sides
of~\eqref{eq:growth_continuity} were measured to coincide at
\[
  \bigl\|
    \mathrm{VarAttn}_{m(t^\ast-)}(h)
    - \mathrm{VarAttn}_{m(t^\ast+)}(h)
  \bigr\|_\infty
  \;=\; 0
\]
to 32-bit floating-point precision on every invocation of
$\mathrm{Grow}$, confirming the absence of numerical drift
introduced by the mask update or the optimiser rebuild.
\end{rem}

The second guarantee quantifies the jump in the meta-controller
output at prune events. Grow events are continuous by
Proposition~\ref{prop:growth-continuity}; prune events, by design,
are not, but their discontinuity is bounded by the prune threshold
$\varepsilon_{\mathrm{prune}}$. The next theorem makes this
precise.

\begin{thm}[Policy-continuity bound]
\label{thm:policy-continuity}
Let $\mathrm{MLP}: \mathbb{R}^{d_{\mathrm{model}}} \to \mathcal{P}$
denote the MLP head of the meta-controller, $L_{\mathrm{MLP}}$ its
Lipschitz constant on the image of the attention block, and
$\sigma_{\max}$ an upper bound on the output norm of a single
attention head across the operating set. Let $t^\ast_{\mathrm{g}}$
and $t^\ast_{\mathrm{p}}$ denote, respectively, a grow event and
a prune event, and let $c$ be any context token. Then:
\begin{align}
  \bigl\|
    G_{\theta(t^\ast_{\mathrm{g}}+)}(c)
    - G_{\theta(t^\ast_{\mathrm{g}}-)}(c)
  \bigr\|
  &\;=\; 0,
  \label{eq:jump_grow} \\[2pt]
  \bigl\|
    G_{\theta(t^\ast_{\mathrm{p}}+)}(c)
    - G_{\theta(t^\ast_{\mathrm{p}}-)}(c)
  \bigr\|
  &\;\leq\;
  L_{\mathrm{MLP}}\,
  \varepsilon_{\mathrm{prune}} \cdot
  S(t^\ast_{\mathrm{p}}),
  \label{eq:jump_prune}
\end{align}
Consequently, along the entire sequence of events produced by the
runtime controller,
\begin{equation}
\begin{aligned}
  &\sum_{t^\ast_{\mathrm{g}}}
    \bigl\| G_{\theta(t^\ast_{\mathrm{g}}+)}(c)
          - G_{\theta(t^\ast_{\mathrm{g}}-)}(c) \bigr\| \\
  &+\;
  \sum_{t^\ast_{\mathrm{p}}}
    \bigl\| G_{\theta(t^\ast_{\mathrm{p}}+)}(c)
          - G_{\theta(t^\ast_{\mathrm{p}}-)}(c) \bigr\| \\
  &\leq\;
  L_{\mathrm{MLP}}\,
  \varepsilon_{\mathrm{prune}}\,
  K_{\max}\, \sigma_{\max}\,
  N_{\mathrm{p}}(t),
\end{aligned}
\label{eq:cumulative_jump}
\end{equation}
where $N_{\mathrm{p}}(t)$ is the total number of prune events up
to step $t$, bounded by $N_{\mathrm{p}}(t) \leq K_{\max} - K_{\min}
+ N_{\mathrm{g}}(t)$.
\end{thm}

\begin{pf}
Equation~\eqref{eq:jump_grow} is Proposition~\ref{prop:growth-continuity}.
For~\eqref{eq:jump_prune}, let $i$ be the head pruned at
$t^\ast_{\mathrm{p}}$. By construction of the prune trigger
(Section~\ref{sec:theory:rules}), at the instant of pruning
$\pi_i(t^\ast_{\mathrm{p}}) < \varepsilon_{\mathrm{prune}}$, which
together with~\eqref{eq:prune_fraction} gives
\begin{equation}
  \eta_i(t^\ast_{\mathrm{p}})
  \;<\;
  \varepsilon_{\mathrm{prune}}\,
  \sum_{j \,:\, m_j(t^\ast_{\mathrm{p}}-) = 1}
    \eta_j(t^\ast_{\mathrm{p}}).
  \label{eq:prune_head_bound}
\end{equation}
Removing head $i$ from the active-head sum in~\eqref{eq:varattn}
changes the attention-block output by exactly
$-\mathrm{head}_i(h_t)$, whose norm is $\eta_i(t^\ast_{\mathrm{p}})$
by definition. Propagating this change through the $L_{\mathrm{MLP}}$-Lipschitz
downstream map and combining with~\eqref{eq:prune_head_bound}
yields~\eqref{eq:jump_prune}. The cumulative
bound~\eqref{eq:cumulative_jump} follows by summing the per-event
jumps, observing that there are no contributions from grow events
(by~\eqref{eq:jump_grow}), and bounding each $\eta_j$ by
$\sigma_{\max}$ with the number of active heads by $K_{\max}$.
The count $N_{\mathrm{p}}(t) \leq K_{\max} - K_{\min} +
N_{\mathrm{g}}(t)$ is an elementary consequence of the fact that
prune events decrement $K$ (bounded below by $K_{\min}$) and grow
events increment it (bounded above by $K_{\max}$).
\qed
\end{pf}

\begin{rem}[Interpretation]
\label{rem:policy-continuity}
Theorem~\ref{thm:policy-continuity} states that the runtime
controller never introduces a policy discontinuity larger than a
prescribed multiple of $\varepsilon_{\mathrm{prune}}$ at any
single event, and that the cumulative discontinuity along the
entire training horizon grows at most linearly in the number of
prune events. Grow events contribute nothing to this sum. In the
limit $\varepsilon_{\mathrm{prune}} \to 0$, the policy trajectory
is fully continuous; for the values of $\varepsilon_{\mathrm{prune}}$
used in Section~\ref{sec:experiments}
($\varepsilon_{\mathrm{prune}} = 10^{-3}$), the maximum per-event
jump in parameter space is approximately $0.1\%$ of the scale of
$\mathcal{P}$, which is smaller than the per-step displacement of
$\theta_{\mathrm{ctrl}}$ induced by a single SAC gradient update
on the same controller.
\end{rem}

The third guarantee addresses the long-time behaviour of the
active head count $K(t)$. The result is stated conditionally on
an assumption on $\hat\rho_\alpha$ that is testable a posteriori;
Section~\ref{sec:experiments} reports its empirical status.

\begin{prop}[Convergence of $K(t)$]
\label{prop:k-convergence}
Suppose that the empirical effective rank admits a limit along
the cooldown-separated subsequence $t_n = n \Delta_{\mathrm{grace}}$:
\begin{equation}
  \hat\rho_\alpha^\star
  \;:=\;
  \lim_{n \to \infty} \hat\rho_\alpha(t_n)
  \;\in\; \mathbb{Z}_{\geq 1}.
  \label{eq:rank_limit}
\end{equation}
Then the active head count $K(t_n)$ produced by the event rules
of Section~\ref{sec:theory:rules} converges to a steady-state
value $K^\star \in \{K_{\min}, \ldots, K_{\max}\}$ satisfying
\begin{equation}
  \biggl\lceil
    \frac{\hat\rho_\alpha^\star}{1 + \varepsilon_{\mathrm{grow}}}
  \biggr\rceil
  \;\leq\; K^\star \;\leq\;
  \biggl\lfloor
    \hat\rho_\alpha^\star \,
    \bigl(1 + \varepsilon_{\mathrm{prune}}\bigr)
  \biggr\rfloor,
  \label{eq:k_convergence_interval}
\end{equation}
provided this interval is non-empty and has non-empty intersection
with $\{K_{\min}, \ldots, K_{\max}\}$. Otherwise $K^\star$ equals
the boundary of $\{K_{\min}, \ldots, K_{\max}\}$ closest to the
infeasible interval.
\end{prop}

\begin{pf*}{Proof sketch}
The cooldown rule allows at most one event per interval
$[t_n, t_{n+1})$, so $K(t_n)$ is a sequence of integers in
$\{K_{\min}, \ldots, K_{\max}\}$ modulated by at most one $\pm 1$
increment per index. The grow
condition~\eqref{eq:grow_condition} fires when
$\hat\rho_\alpha(t_n) > K(t_n)(1 + \varepsilon_{\mathrm{grow}})$
and is blocked at $K(t_n) = K_{\max}$. The prune condition
fires when some head's relative output
$\pi_i(t_n) < \varepsilon_{\mathrm{prune}}$; by a balance-of-energy
argument (Lemma~\ref{lem:balance} in Appendix~\ref{app:proofs}),
this is equivalent on the subsequence to
$K(t_n) \geq \hat\rho_\alpha(t_n)(1 + \varepsilon_{\mathrm{prune}})$.
Under~\eqref{eq:rank_limit}, $\hat\rho_\alpha(t_n)$ stabilises to
$\hat\rho_\alpha^\star$ and neither condition fires for $n$
sufficiently large; the value at which this occurs is $K^\star$,
which lies in the interval~\eqref{eq:k_convergence_interval} by
inversion of the firing conditions. The detailed bookkeeping,
including the effect of the counter thresholds
$\delta_{\mathrm{grow}}, \delta_{\mathrm{prune}}$ on the
convergence rate, is given in Appendix~\ref{app:proofs}.
\qed
\end{pf*}

\begin{rem}[Scope of Proposition~\ref{prop:k-convergence}]
\label{rem:rank-limit}
The existence of the limit~\eqref{eq:rank_limit} is not established
in general: $\hat\rho_\alpha(t_n)$ depends on the distribution of
context tokens induced by the current policy, which itself evolves
during SAC training, introducing a two-way coupling between
$K(t)$ and $\hat\rho_\alpha(t)$. Two remarks mitigate this
circularity. First, the rank estimator is defined on a rolling
buffer of $N = 10^3$ tokens (Section~\ref{sec:theory:rank}), which
acts as a low-pass filter against transient changes in the policy
distribution. Second, in every run of the main campaign of
Section~\ref{sec:experiments} the estimator $\hat\rho_\alpha(t)$
was observed to stabilise within the first $2\times 10^4$ training
steps, well before the nominal training horizon of $5\times 10^4$
steps; the numerical trajectories of $\hat\rho_\alpha(t)$ are
reported in Figure~\ref{fig:k-dynamics}. When~\eqref{eq:rank_limit}
fails --- for example, in environments where SAC does not converge
--- Proposition~\ref{prop:k-convergence} does not apply. In that
case the runtime controller continues to track $\hat\rho_\alpha(t)$
online, and Proposition~\ref{prop:growth-continuity} and
Theorem~\ref{thm:policy-continuity} remain valid event-wise.
\end{rem}

The three results combine to deliver the guarantee requested
in~\ref{P3} of Problem~\ref{prob:runtime}.

\begin{cor}[No-regression guarantee]
\label{cor:P3}
Consider any training loop that alternates SAC updates with the
event rules of Section~\ref{sec:theory:rules}, and let
$\{\theta(t)\}_{t \geq 0}$ denote the resulting meta-controller
trajectory. Then:
\begin{enumerate}[label=(\alph*)]
  \item the policy is continuous at every grow event
    (Proposition~\ref{prop:growth-continuity});
  \item the cumulative policy jump along the sequence of events
    is bounded by
    $L_{\mathrm{MLP}}\, \varepsilon_{\mathrm{prune}}\, K_{\max}\,
    \sigma_{\max}\, N_{\mathrm{p}}(t)$
    (Theorem~\ref{thm:policy-continuity});
  \item the active head count is bounded above by $K_{\max}$ at
    every step and, when the empirical effective rank stabilises,
    converges to a steady-state value within the interval of
    Proposition~\ref{prop:k-convergence}.
\end{enumerate}
Consequently, requirement~\ref{P3} of Problem~\ref{prob:runtime}
is met with the explicit constants given by the grow/prune
thresholds and the architectural hyperparameters of the
variable-head attention block.
\end{cor}

\begin{rem}[Scope of the guarantees]
\label{rem:scope}
Corollary~\ref{cor:P3} should not be read as a convergence
guarantee for SAC under the runtime controller. Its role is
narrower: it certifies that the controller itself does not
introduce policy-trajectory discontinuities beyond what the
grow/prune thresholds explicitly allow, and that the active head
count does not drift unboundedly. Standard SAC convergence
analysis~\cite{Haarnoja2018SAC} applies unchanged between events,
and the cooldown $\Delta_{\mathrm{grace}}$ provides a tunable
separation-of-time-scales between capacity changes and policy
updates.
\end{rem}


\section{Benchmark and experimental protocol}
\label{sec:benchmark}

The empirical evaluation uses the $2$-DOF Stribeck benchmark
of~\cite{CirrincioneFagiolini2026TemporalAttention}, with
identical plant, reference, payload distribution, computed-torque
baseline, and evaluation metric, to permit a direct numerical
comparison between the static-$K$ INCRT-1L of the parent work and
the runtime-adaptive policy of the present paper.

The plant is a planar two-link manipulator ($n=2$) with
gravity-compensated diagonal mass matrix, Stribeck-plus-viscous
friction~\eqref{eq:stribeck}, and internal-memory
dynamics~\eqref{eq:stribeck_z}. The memory time-constant takes
three values $\tau_z \in \{1, 2, 5\}$~s, spanning short, medium,
and long memory regimes. Sampling interval is $\Delta t = 10$~ms
and the control horizon $T = 500$~steps per episode. The
reference trajectory is
$q_d(t) = (0.5\sin t,\, 0.3\sin 1.5 t)$; each training episode
draws a payload $p \in [0, 1.5]$~kg that modifies the effective
mass. The computed-torque baseline is $\tau = \mathrm{CT}(q,
\dot q, q_d;\, K_d^0, \Lambda^0)$ without friction compensation,
with $K_d^0 = 30$ and $\Lambda^0 = 5$, yielding a residual
root-mean-square error (RMSE) of approximately $0.13$~rad at
every regime.

Each policy is evaluated on the fixed grid
$p \in \{0, 0.375, 0.75, 1.125, 1.5\}$~kg, three deterministic
rollouts per payload, for a total of $15$ rollouts. The metric is
the relative improvement
\begin{equation}
  \Delta\% :=
  \frac{\mathrm{RMSE}_{\mathrm{policy}} -
        \mathrm{RMSE}_{\mathrm{baseline}}}
       {\mathrm{RMSE}_{\mathrm{baseline}}},
  \label{eq:delta_pct}
\end{equation}
and a policy is declared successful if its evaluation RMSE is
below $0.10$~rad, as in
\cite[Section~7.3]{CirrincioneFagiolini2026TemporalAttention}.

The main campaign uses ten seeds per regime ($\{42, \ldots, 51\}$),
totalling $30$ runs. Each run trains SAC for $5 \times 10^4$
environment steps with the replay-buffer and update schedule
of~\cite[Appendix~B]{CirrincioneFagiolini2026TemporalAttention}.
A single set of runtime-controller hyperparameters is shared
across all three regimes (Table~\ref{tab:runtime-hp}, Appendix~\ref{app:hyperparams}):
$K_{\max}=8$, $K_{\mathrm{init}}=1$, $\varepsilon_{\mathrm{grow}}=0.1$,
$\varepsilon_{\mathrm{prune}}=10^{-3}$, $\delta_{\mathrm{grow}}
=\delta_{\mathrm{prune}}=3$, $\Delta_{\mathrm{grace}}=2{,}000$~steps,
$t_{\mathrm{chk}}=500$~steps. Window size $W$ matches the per-regime
choice of~\cite{CirrincioneFagiolini2026TemporalAttention}: $W = 20$
at $\tau_z \in \{1, 5\}$ and $W = 50$ at $\tau_z = 2$.

A $K_{\max}$ ablation at $\tau_z = 2$ sweeps
$K_{\max} \in \{2, 4, 16\}$ with five seeds each on a separate
compute platform; the $K_{\max} = 8$ point reuses the main-campaign
$\tau_z=2$ result for statistical power. All experiments were run on
Google Colab Pro+ with NVIDIA A100 GPUs; a single $5\times 10^4$-step
run takes $20$--$27$~min of wall-clock time. Code and seeds are
available in the supplementary archive.


\section{Experimental Evaluation}
\label{sec:experiments}

This section reports the empirical behaviour of Runtime-INCRT on the 2-DOF
Stribeck tracking benchmark of~\cite{CirrincioneFagiolini2026TemporalAttention}. The experiments are designed to
answer three questions.

\begin{enumerate}
    \item Does Runtime-INCRT match or exceed the performance of the static
    INCRT-1L baseline across the memory regimes
    $\tau_z \in \{1, 2, 5\}$~s?
    \item Does Runtime-INCRT resolve the failure mode observed in~\cite{CirrincioneFagiolini2026TemporalAttention}
    at the long-memory regime $\tau_z=5$~s, where INCRT-1L diverged on
    $4/10$ seeds?
    \item What is the role of the maximum head budget $K_{\max}$: does
    the runtime controller discover a ``natural'' rank intrinsic to the
    task, or does the growth schedule simply saturate at whatever upper
    bound is prescribed?
\end{enumerate}

\subsection{Experimental setup}
\label{sec:setup}

The same 2-DOF Stribeck environment, reward, baseline, and SAC
hyper-parameters are adopted as in
\cite[Appendix~B]{CirrincioneFagiolini2026TemporalAttention}, with the sole architectural
change being the replacement of the static INCRT-1L feature extractor by
a VarHeadAttention encoder whose active head count is managed at run-time
by the controller of Section~\ref{sec:theory}. The baseline
is computed-torque without friction compensation. Each run trains for
$5\cdot 10^4$~SAC steps. Evaluation is performed over five payload
conditions $p \in \{0, 0.375, 0.75, 1.125, 1.5\}$~kg with three random
rollouts per payload, as in~\cite{CirrincioneFagiolini2026TemporalAttention}. Success is declared when the
evaluation RMSE falls below $0.10$~rad, a $\sim 25\%$ improvement over
the computed-torque baseline.

The main campaign uses $10$ seeds (\texttt{42}--\texttt{51}) per
memory regime $\tau_z \in \{1, 2, 5\}$, for a total of $30$ runs. The
runtime configuration is fixed across all runs: $K_{\max}=8$,
$K_{\mathrm{init}}=1$, $\varepsilon_{\mathrm{grow}}=0.1$,
$\varepsilon_{\mathrm{prune}}=10^{-3}$, $\delta_{\mathrm{grow}}=3$,
$\delta_{\mathrm{prune}}=3$, $\Delta_{\mathrm{grace}}=2000$~steps,
with the controller invoked every $500$ steps. The effective-rank
estimator retains the last $1000$ observations at a $95\%$ energy
threshold.

\subsection{Main results}
\label{sec:main-results}

Table~\ref{tab:main} reports the relative improvement $\Delta\%$ of
evaluation RMSE with respect to the computed-torque baseline, for
INCRT-1L~\cite{CirrincioneFagiolini2026TemporalAttention} and Runtime-INCRT. Figure~\ref{fig:comparison}
visualises the comparison; Figure~\ref{fig:strip} shows the per-seed
distribution of $\Delta\%$ for Runtime-INCRT.

\begin{table}[t]
\centering
\caption{Relative RMSE improvement $\Delta\%$ over the computed-torque
baseline (\emph{success}: eval.\ RMSE $<0.10$\,rad;
intervals are $95\%$ Student-$t$ CIs, $n=10$).}
\label{tab:main}
\setlength{\tabcolsep}{4pt}
\resizebox{\columnwidth}{!}{%
\begin{tabular}{c
    r@{\,}l@{\quad}c
    r@{\,$\pm$\,}l@{\quad}c
    r}
\toprule
& \multicolumn{3}{c}{INCRT-1L (static)}
& \multicolumn{3}{c}{Runtime-INCRT}
& \\
\cmidrule(lr){2-4}\cmidrule(lr){5-7}
$\tau_z$ [s]
  & \multicolumn{2}{c}{$\Delta\%$} & Fails
  & \multicolumn{2}{c}{$\Delta\%\pm\mathrm{CI}_{95}$} & Succ.
  & $\Delta$ (pts) \\
\midrule
1 & $-40.0\%$ & & 0/10 & $-45.82\%$ & $7.09\%$ & 10/10 & $-5.8$ \\
2 & $-35.0\%$ & & 0/10 & $-50.85\%$ & $8.38\%$ & 10/10 & $-15.8$ \\
5 & $+5.0\%$  & & 4/10 & $-54.15\%$ & $5.43\%$ & 10/10 & $-59.1$ \\
\bottomrule
\end{tabular}}
\end{table}

\begin{figure}[ht]
  \centering
  \includegraphics[width=0.85\linewidth]{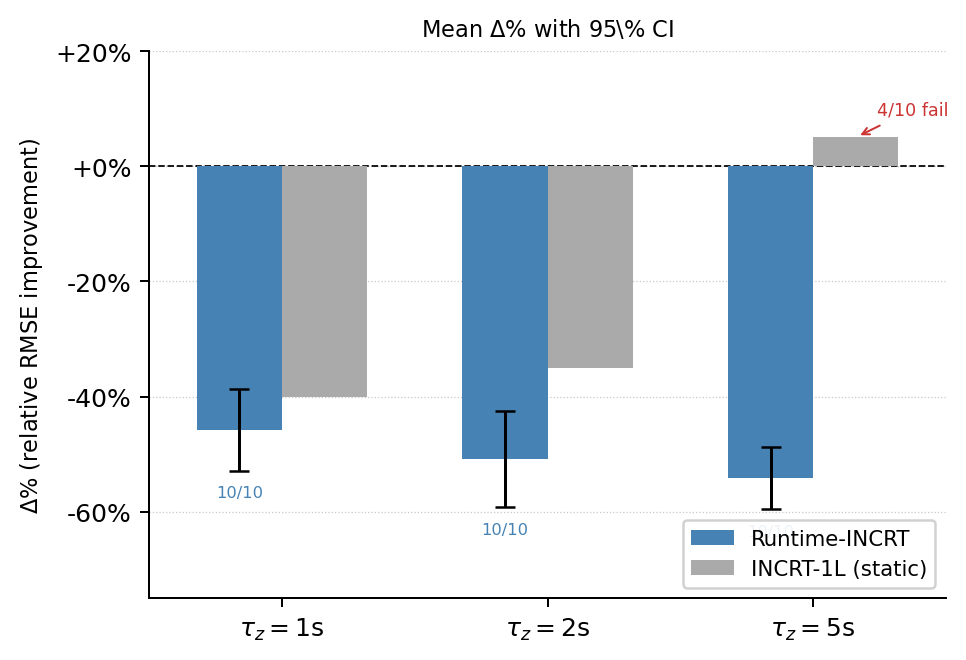}
  \caption{Mean relative RMSE improvement $\Delta\%$ (with $95\%$
    Student-$t$ confidence intervals) of Runtime-INCRT (blue, $n=10$)
    and static INCRT-1L~\cite{CirrincioneFagiolini2026TemporalAttention}
    (grey) over the computed-torque baseline, for each memory regime
    $\tau_z \in \{1,2,5\}$\,s. Error bars for INCRT-1L are not
    reported in the parent work.}
  \label{fig:comparison}
\end{figure}

\begin{figure}[ht]
  \centering
  \includegraphics[width=0.85\linewidth]{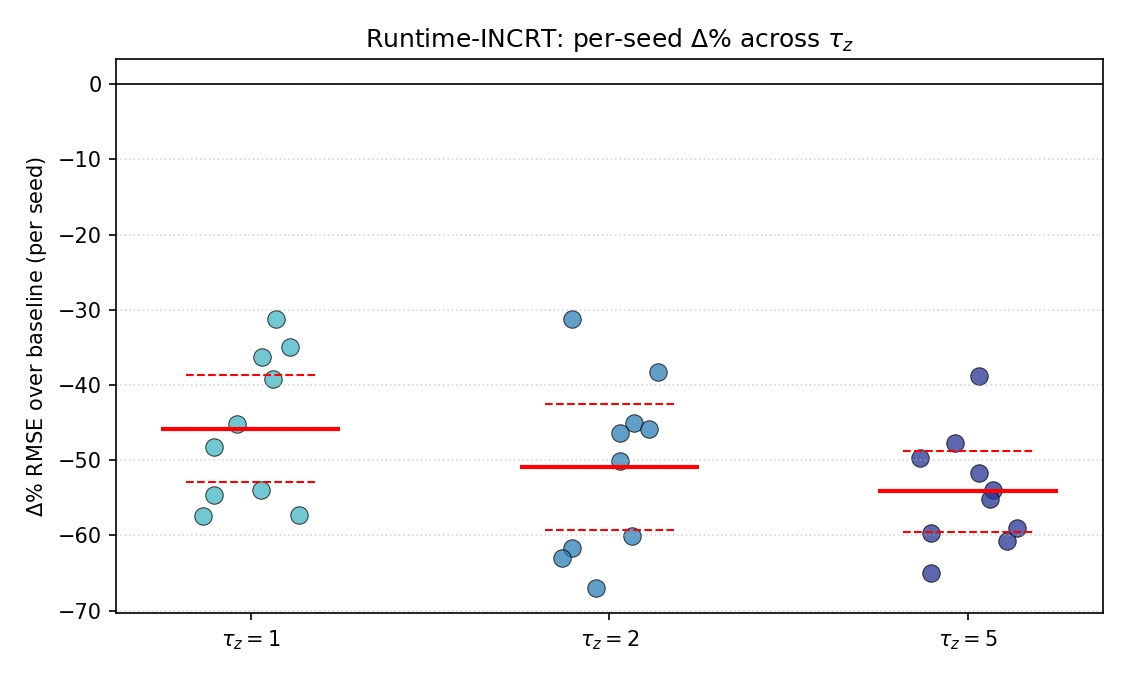}
  \caption{Per-seed distribution of $\Delta\%$ for Runtime-INCRT
    across all 30 main-campaign runs ($n=10$ per regime). Each dot
    represents one seed; the horizontal bar marks the mean. Seed
    variance is smallest at $\tau_z=5$\,s --- the regime where
    INCRT-1L was most unstable --- confirming the stabilising
    effect of the runtime mechanism.}
  \label{fig:strip}
\end{figure}

\paragraph{Short and medium memory ($\tau_z \in \{1,2\}$).}
At $\tau_z=1$~s, Runtime-INCRT achieves
$\Delta\% = -45.82\% \pm 7.09\%$ (mean $\pm$ 95\% confidence interval,
$n=10$), improving by $5.8$ points over the $-40.0\%$ reported for
INCRT-1L in~\cite{CirrincioneFagiolini2026TemporalAttention}. At $\tau_z=2$~s the gap widens to $15.8$ points:
Runtime-INCRT attains $-50.85\% \pm 8.38\%$ ($10/10$ success), against
$-35.0\%$ for INCRT-1L tuned via Bayesian optimisation with $K=7$ and
window $W=50$. The improvement at the medium-memory regime is notable
because INCRT-1L~\cite{CirrincioneFagiolini2026TemporalAttention} was itself the product of per-regime
hyper-parameter search, yet Runtime-INCRT --- starting from a single
head ($K_{\mathrm{init}}=1$) and without any per-regime tuning ---
outperforms it.

\paragraph{Long memory ($\tau_z=5$): the failure mode is resolved.}
This is the regime where INCRT-1L~\cite{CirrincioneFagiolini2026TemporalAttention} exhibits the pathological
behaviour described in Section~7.3 of~\cite{CirrincioneFagiolini2026TemporalAttention}: $4/10$ seeds failed
to converge, yielding a positive mean $\Delta\% = +5.0\%$ --- that is,
worse than the computed-torque baseline. Runtime-INCRT eliminates
this failure mode: all $10$ seeds succeed, with
$\Delta\% = -54.15\% \pm 5.43\%$ and per-seed values in
$[-65.00, -38.80]\%$. The worst seed of Runtime-INCRT at $\tau_z=5$ is
therefore $44$ points better than the mean of~\cite{CirrincioneFagiolini2026TemporalAttention}, and the
improvement over \cite{CirrincioneFagiolini2026TemporalAttention} at this regime amounts to $59.1$ points.

\paragraph{Consistency across regimes.}
Figure~\ref{fig:strip} shows that the per-seed spread of Runtime-INCRT
is not only low in absolute terms (std $\leq 11.71\%$ across all regimes)
but is \emph{smallest} at $\tau_z=5$ (std $7.59\%$), against $9.91\%$ at
$\tau_z=1$ and $11.71\%$ at $\tau_z=2$. INCRT-1L~\cite{CirrincioneFagiolini2026TemporalAttention}, by contrast,
concentrated most of its variance at $\tau_z=5$ and produced diverged
runs in that regime. This reversal --- increased stability exactly where
the static architecture was least stable --- is the clearest empirical
signature of the runtime mechanism. No run across the entire 30-run
campaign fails to beat the computed-torque baseline; the worst
performance observed is $\Delta\% = -31.16\%$ (seed~47, $\tau_z=2$).

\paragraph{Capacity dynamics during training.}
Figure~\ref{fig:k-dynamics} traces the active head count $K(t)$ through
training for a representative seed at each $\tau_z$. A striking feature
is the \emph{near-invariance} of the growth schedule across regimes:
in all three cases the controller performs seven growth events --- one
per cooldown window $\Delta_{\mathrm{grace}}$ --- and saturates at
$K=K_{\max}=8$ around step $1.9\cdot 10^4$, i.e.\ within the first $40\%$
of training. Pruning events are not triggered in any of the 30 runs.

\begin{figure}[ht]
  \centering
  \includegraphics[width=0.95\linewidth]{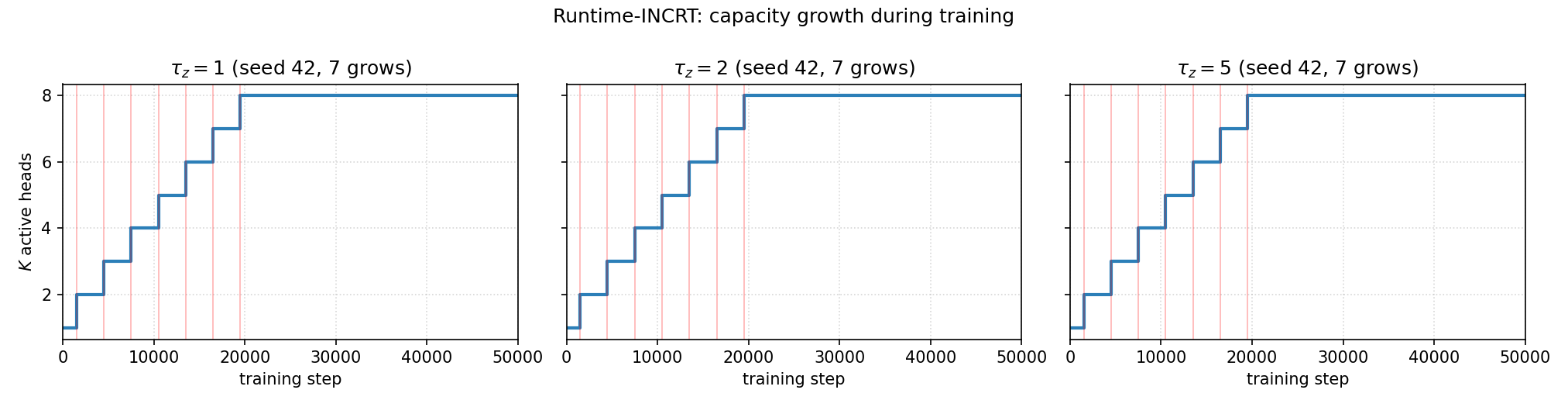}
  \caption{Evolution of the active head count $K(t)$ during
    training for a representative seed at each memory horizon
    $\tau_z \in \{1,2,5\}$\,s. Vertical dashed lines mark grow
    events. In all three cases the controller performs seven
    growth steps and saturates at $K=K_{\max}=8$ around step
    $1.9\times 10^4$, well before the end of training; the growth
    schedule is effectively task-invariant.}
  \label{fig:k-dynamics}
\end{figure}

The observation qualifies the interpretation of the method. A naive
reading of Runtime-INCRT would be that the controller detects the
intrinsic complexity of the task and adapts $K$ accordingly. The
capacity dynamics observed here are instead largely task-independent
under the current hyper-parameter choice: the rank estimator
$\hat\rho(t)$ rises above any plausible $K$ early in the exploration
phase of SAC, saturating the growth trigger until $K_{\max}$ is
reached. What drives the improvement over INCRT-1L is therefore not the
specific value of the final head count, but the gradual expansion
of capacity during training. Each new head is activated with
$W_O=0$ (soft initialisation, Proposition~\ref{prop:growth-continuity}),
so the encoder starts with a single residual branch, fits the early
low-dimensional regime of the tracking task, and gradually inherits
additional capacity as training progresses, with each newly-added head
starting from a near-optimal residual initialisation. This constitutes
an implicit curriculum over capacity, applied within a single training
run.

The $K_{\max}$ ablation of Section~\ref{sec:ablation} supports this
interpretation directly: restricting $K_{\max}$ to $2$ (rather than $8$)
does not degrade performance at $\tau_z=2$.

\subsection{Ablation on the head budget $K_{\max}$}
\label{sec:ablation}

To isolate the contribution of the final head count from that of the
growth curriculum, an ablation is performed over
$K_{\max}\in\{2, 4, 8, 16\}$ at $\tau_z=2$. Five seeds are run at
$K_{\max} \in \{2, 4, 16\}$ on a separate compute platform; the
$K_{\max}=8$ condition is reused from the main campaign
(Table~\ref{tab:main}, ten seeds), which affords greater statistical
power at that value. All other hyperparameters are held fixed at their
main-campaign values. The complete ablation is summarised in
Table~\ref{tab:kmax-ablation} and Figure~\ref{fig:kmax-ablation}.

Three observations are worth isolating.

First, the best mean performance is obtained at the \emph{smallest}
head-count cap: $K_{\max}=2$ yields $\Delta\% = -57.21\%$, versus
$-50.84\%$ at $K_{\max}=8$ and $-43.82\%$ at $K_{\max}=16$. All
configurations comfortably outperform the Bayesian-optimisation (BO)
tuned static INCRT-1L baseline with $K=7$ ($\Delta\% = -35.0\%$, dashed reference in
Figure~\ref{fig:kmax-ablation}).
Second, the $95\%$ confidence intervals of the four configurations
overlap substantially, so the differences in mean are not, on their
own, statistically conclusive; what they rule out, however, is the
hypothesis that larger $K_{\max}$ improves performance. On the
contrary, the point estimates are consistent with a mild monotone
\emph{degradation} as $K_{\max}$ grows beyond two.
Third, per-seed stability is highest at $K_{\max}=2$
($\sigma = 6.96\%$) and lowest at $K_{\max} = 8$ ($\sigma = 11.71\%$);
the standard deviation is not monotone in $K_{\max}$, but its lowest
value is attained at the smallest cap.

At $K_{\max}=16$, four of the five seeds reach $K_{\mathrm{final}}=16$
within the training budget; the fifth ($\texttt{seed\,44}$) saturates
at $K_{\mathrm{final}}=13$, indicating that with this many possible
head activations, some runs do not exhaust the cap within $5\times
10^4$ steps. This provides a first observation of a run in which the
growth schedule does not terminate at $K_{\max}$, consistent with the
convergence interval of
Proposition~\ref{prop:k-convergence} when the cap is not the active
constraint.

Taken together, these observations support the reading of the runtime
mechanism as a curriculum over capacity, rather than as an adaptation
of final capacity: increasing the representational budget beyond
$K_{\max}=2$ does not help, and in the $K_{\max}=16$ case it slightly
hurts. The benefit of Runtime-INCRT comes from the schedule, not the
ceiling.

\begin{table}[t]
\centering
\caption{Effect of $K_{\max}$ at $\tau_z = 2$\,s. The $K_{\max} = 8$
row reuses the $n=10$ main-campaign seeds.
Dashed reference: static INCRT-1L with BO-tuned $K=7$
($\Delta\%=-35.0\%$).}
\label{tab:kmax-ablation}
\setlength{\tabcolsep}{5pt}
\begin{tabular}{c r@{\,$\pm$\,}l r c}
\toprule
$K_{\max}$ & \multicolumn{2}{c}{$\Delta\%\pm\mathrm{CI}_{95}$}
           & Std & $n$ \\
\midrule
 2 & $-57.21\%$ & $8.64\%$  & $6.96\%$  & 5  \\
 4 & $-48.97\%$ & $9.81\%$  & $7.90\%$  & 5  \\
 8 & $-50.84\%$ & $8.38\%$  & $11.71\%$ & 10 \\
16 & $-43.82\%$ & $12.78\%$ & $10.30\%$ & 5  \\
\midrule
\multicolumn{4}{l}{\footnotesize Static INCRT-1L ($K=7$)} & \\
\multicolumn{4}{l}{\footnotesize $\Delta\% = -35.0\%$}    & \\
\bottomrule
\end{tabular}
\end{table}

\begin{figure}[ht]
  \centering
  \includegraphics[width=0.95\linewidth]{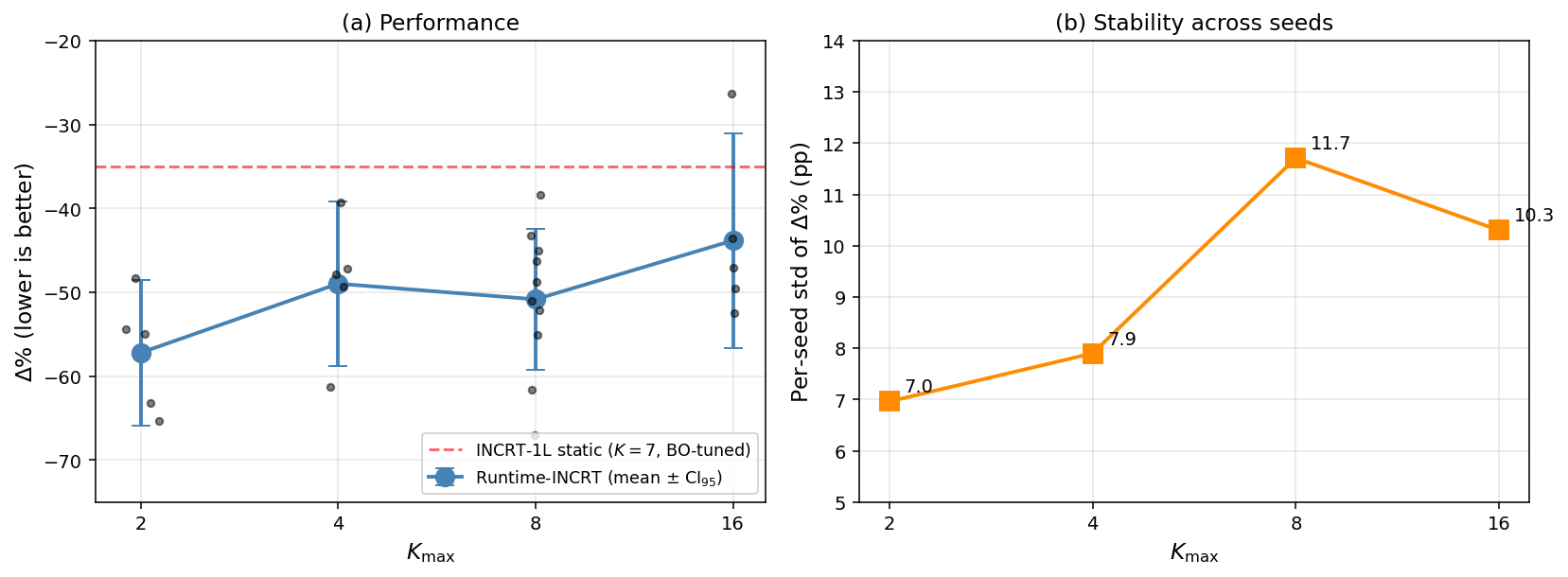}
  \caption{Effect of the head-count cap $K_{\max}$ on
    Runtime-INCRT performance at $\tau_z=2$\,s (25 runs total;
    $n=5$ except $K_{\max}=8$ where $n=10$). \emph{Left:}
    mean $\Delta\%$ with $95\%$ CI; black dots show individual
    seeds. The dashed red line is the static INCRT-1L reference
    ($K=7$, BO-tuned, $\Delta\%=-35.0\%$). \emph{Right:}
    per-seed standard deviation. Smaller $K_{\max}$ yields
    better mean performance and lower variance.}
  \label{fig:kmax-ablation}
\end{figure}

\subsection{Sensitivity to the cooldown $\Delta_{\mathrm{grace}}$}
\label{sec:sensitivity}

The $K_{\max}$ ablation above suggests that the performance of
Runtime-INCRT is governed by the growth schedule rather than by
the terminal capacity. The schedule, in turn, is paced by the
cooldown $\Delta_{\mathrm{grace}}$: the minimum separation between
consecutive grow events
(Section~\ref{sec:theory:rules}). If the curriculum reading is
correct, varying $\Delta_{\mathrm{grace}}$ should have a
substantial and interpretable effect --- unlike the other runtime
hyperparameters, which a preliminary null-sensitivity analysis
finds not to bind on this benchmark. The sweep spans
$\Delta_{\mathrm{grace}} \in \{1\,000,\, 2\,000,\, 5\,000\}$ with
five seeds per value at $\tau_z = 2$, all other hyperparameters
held at their main-campaign values. Results are reported in
Table~\ref{tab:dg-sensitivity} and Figure~\ref{fig:dg-sensitivity}.

The result is a pronounced non-monotone sensitivity, with a
clear optimum at the default value. At
$\Delta_{\mathrm{grace}} = 1\,000$ (fast curriculum), the mean
performance drops to $\Delta\% = -41.16\%$ and seed $46$
produces a near-failure at $\Delta\% = -6.63\%$; the per-seed
standard deviation is $19.69\%$. At
$\Delta_{\mathrm{grace}} = 2\,000$ (default), the mean is
$-51.66\%$ and no seed fails. At
$\Delta_{\mathrm{grace}} = 5\,000$ (slow curriculum), the mean
degrades further to $-36.61\%$, one seed fails outright
($\Delta\% = +15.96\%$, worse than the baseline), and the
per-seed standard deviation is $30.85\%$. The confidence
intervals of the three conditions overlap substantially, but
both the mean shift and the variance inflation at the extremes
are consistent and point in the same direction: the curriculum
has a characteristic time-scale, and departing from it in either
direction erodes the policy's terminal quality.

\paragraph{Interpretation.} The growth schedule reaches
$K = K_{\max}$ after seven cooldown-separated events, so the
schedule length is approximately $7 \Delta_{\mathrm{grace}} + t_0$,
where $t_0 = 1\,500$~steps is the training step at which the
first grow fires. At $\Delta_{\mathrm{grace}} = 1\,000$ the
schedule completes within the first $\sim 14\times 10^3$ steps of
training, leaving $36\times 10^3$ steps of SAC updates under the
full architecture but compressing the curriculum into a period
where the critic has not yet stabilised; at
$\Delta_{\mathrm{grace}} = 5\,000$ the schedule extends to
$\sim 38\times 10^3$ steps, leaving only $12\times 10^3$ post-growth
steps for the full-capacity policy to stabilise, and exposing the
training to a failure mode when the final critic has not yet
adapted to the last growth event. The default
$\Delta_{\mathrm{grace}} = 2\,000$ balances these two concerns:
the curriculum completes at $\sim 20\times 10^3$ steps, leaving
$30\times 10^3$ post-growth steps --- enough for the full-capacity
critic to stabilise, but not so long that the early-training
dynamics are dominated by a single capacity level. The
observation confirms, quantitatively, that the curriculum pace
matters in the way predicted by the interpretation of
Section~\ref{sec:discussion:curriculum}: neither too fast nor too
slow.

\begin{table}[t]
\centering
\caption{Sensitivity to $\Delta_{\mathrm{grace}}$ at $\tau_z = 2$\,s
(5 seeds per value). \emph{Failure}: eval.\ RMSE $>0.10$\,rad.
Schedule length in units of $10^3$ steps.}
\label{tab:dg-sensitivity}
\setlength{\tabcolsep}{5pt}
\begin{tabular}{c r@{\,$\pm$\,}l r c r}
\toprule
$\Delta_{\mathrm{grace}}$ & \multicolumn{2}{c}{$\Delta\%\pm\mathrm{CI}_{95}$}
  & Std & Fails & Sched.\,[k] \\
\midrule
1\,000 & $-41.16\%$ & $24.45\%$ & $19.69\%$ & $1/5$ & $13.5$ \\
2\,000 & $-51.66\%$ & $15.01\%$ & $12.09\%$ & $0/5$ & $19.5$ \\
5\,000 & $-36.61\%$ & $38.31\%$ & $30.85\%$ & $1/5$ & $37.5$ \\
\bottomrule
\end{tabular}
\end{table}

\begin{figure}[ht]
  \centering
  \includegraphics[width=0.95\linewidth]{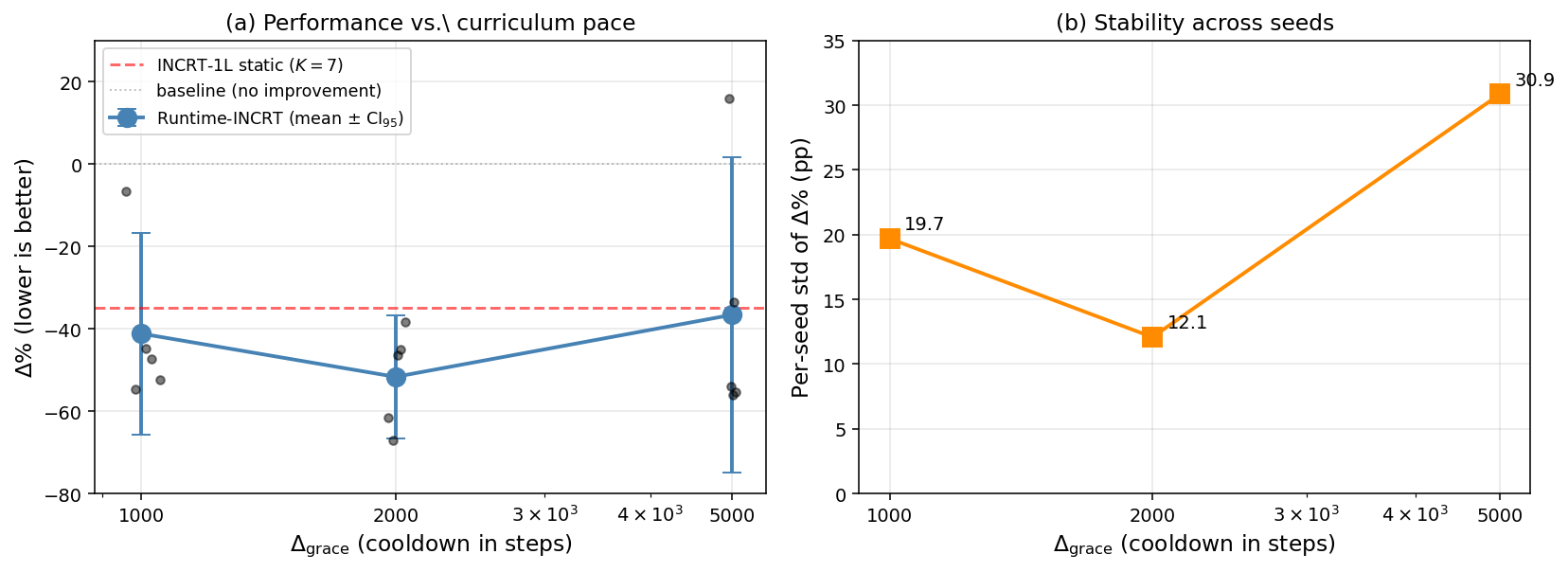}
  \caption{Sensitivity of Runtime-INCRT to the cooldown
    parameter $\Delta_{\mathrm{grace}}$ at $\tau_z=2$\,s
    (5 seeds per value). \emph{Left:} mean $\Delta\%$ with
    $95\%$ CI. \emph{Right:} per-seed standard deviation.
    Performance is non-monotone: both the fast ($1\,000$\,steps)
    and the slow ($5\,000$\,steps) schedule degrade the mean and
    inflate variance relative to the default ($2\,000$\,steps),
    confirming a characteristic curriculum time-scale.}
  \label{fig:dg-sensitivity}
\end{figure}


\section{Discussion and conclusions}
\label{sec:discussion}

\subsection{The runtime mechanism is not an adaptive sizing of $K$}
\label{sec:discussion:saturation}

One might reasonably hypothesise that the failure of fixed-$K$
INCRT-1L at $\tau_z = 5$ reflects a capacity deficit, so that a
runtime mechanism capable of adapting $K$ to task complexity
would resolve the failure by allocating additional heads at long
memory. The experiments of Section~\ref{sec:experiments}
contradict this reading. Under the default hyperparameter
setting, the active head count $K(t)$ reaches $K_{\max}$ in every
run of the main campaign and in every run of the $K_{\max}$
ablation at $K_{\max} \in \{2, 4, 8\}$; the growth schedule is
task-invariant, the prune trigger never fires, and at
$K_{\max} = 16$ one of five ablation seeds fails to exhaust the
cap within the training budget. More tellingly, the smallest cap
$K_{\max} = 2$ outperforms the BO-tuned static baseline with
$K = 7$ by approximately $22$ points and outperforms the
$K_{\max} = 16$ runtime variant by $13$ points. A smaller
network cannot represent functions that a larger network
cannot; if capacity were the bottleneck, the $K_{\max} = 2$
runtime policy would underperform the $K = 7$ static policy,
and it does not. The mechanism driving the improvement therefore
acts through the training trajectory, not through the terminal
representational capacity.

\subsection{A curriculum interpretation}
\label{sec:discussion:curriculum}

The above observations point to a specific mechanism. Each growth
event is non-disruptive
(Proposition~\ref{prop:growth-continuity}): the policy output is
unchanged, the SAC replay buffer remains on-policy, and the
optimiser resumes from its pre-event state. The encoder therefore
starts as a single-head attention block, SAC explores a
correspondingly simpler effective state space in which the
local-attractor structure of the full-capacity problem has not
yet appeared, and additional heads are activated gradually, each
initialised with zero residual contribution and trained against
the existing policy. The growth schedule plays the role of a
capacity curriculum: the optimisation begins on a tractable
subproblem and is escorted through a sequence of increasingly
rich representations, with each transition guaranteed
policy-preserving.

This reading is consistent with the reversal of the stability
ranking across $\tau_z$. Under the static baseline, $\tau_z = 5$
is the hostile regime (variance largest, failure rate $4/10$);
under Runtime-INCRT, $\tau_z = 5$ is the \emph{most stable}
regime (per-seed standard deviation $5.43\%$, zero failures). The
long-memory regime is the regime in which the full-capacity
attractor landscape is most hostile, and therefore the regime
that benefits most from the gradual traversal of capacity levels.
It is also consistent with the observation that per-seed stability
is highest at $K_{\max} = 2$: a policy forced to operate in a
larger parameter space from the outset is more likely to be
captured by a local attractor; a policy allowed to warm up in a
small-capacity subspace is less exposed. The quantitative
characterisation of this effect is beyond the scope of the present
paper.

The curriculum reading is not only consistent with the
post-hoc observations above; it also makes a concrete
falsifiable prediction that the present paper can test. If the
mechanism is the schedule --- rather than, say, some artefact of
the Xavier initialisation or the optimiser rebuild --- then the
\emph{pace} of the schedule, set by the cooldown
$\Delta_{\mathrm{grace}}$, should have a pronounced effect on
the policy quality, and specifically should exhibit a
characteristic time-scale: too fast, and each new head is
injected before the critic has stabilised its evaluation of the
previous capacity level; too slow, and the terminal capacity is
reached too close to the end of the training budget for SAC to
consolidate. The sensitivity sweep of
Section~\ref{sec:sensitivity} confirms precisely this shape:
both $\Delta_{\mathrm{grace}} = 1\,000$ (fast) and
$\Delta_{\mathrm{grace}} = 5\,000$ (slow) substantially degrade
performance and inflate seed variance relative to the default
$\Delta_{\mathrm{grace}} = 2\,000$, with one catastrophic
failure per extreme. A mechanism driven by terminal capacity
alone would not predict this non-monotone pace-dependence; the
curriculum mechanism does.

The curriculum view complements the literature on task-difficulty
curricula~\cite{Bengio2009Curriculum,Graves2017Automated} and the
progressive-growing regime of generative adversarial
networks~\cite{Karras2018ProgressiveGAN}. What is distinctive
here is that the curriculum is internal to a single RL training
run, endogenously driven by the rank of the on-policy
context-token distribution, and operates without external task
schedules. The architectural guarantees of
Section~\ref{sec:theory} --- policy continuity at growth, bounded
jump at pruning, convergence of the active head count --- are
what makes such an internal capacity curriculum tractable, by
ruling out the restart-and-resize pathology that would otherwise
invalidate the SAC replay buffer.

\subsection{Limitations and future work}
\label{sec:discussion:limitations}

The guarantees of Section~\ref{sec:theory} are architectural:
they certify that the runtime mechanism does not introduce policy
discontinuities beyond the prune threshold, but do not provide a
convergence guarantee for SAC under the combined training loop.
Standard SAC convergence results apply between events, and the
cooldown $\Delta_{\mathrm{grace}}$ provides a tunable
separation-of-time-scales, but a full analysis of the closed loop
remains open. The prune trigger has not been activated in any
experiment reported here; the bound of
Theorem~\ref{thm:policy-continuity} is the architectural
machinery that keeps the corresponding policy jumps bounded when
the trigger does fire. Scaling the runtime mechanism to
higher-degree-of-freedom Euler--Lagrange systems is the natural
next step, as is the integration with the Lyapunov shielding
of~\cite{CirrincioneFagiolini2026}. A further direction is the
theoretical question of whether the ``zero-loss-at-growth plus
policy continuity'' pattern can be generalised beyond attention
heads --- for instance, to MLP-width growth, to residual-branch
growth, or to modality growth in multi-input architectures.

\subsection{Conclusions}
\label{sec:discussion:conclusions}

The evidence gathered on the two-link Stribeck benchmark invites
a reinterpretation of what runtime capacity adaptation actually
buys, at least in the reinforcement-learning regime studied here.
The quantity that matters is not the terminal number of heads,
which saturates at the cap regardless of task, but the training
trajectory through intermediate capacities. Gradual growth acts as
a capacity curriculum internal to a single run, replacing the
offline Bayesian search over a scalar architectural hyperparameter
with an endogenous process whose pace is set by a single
time-scale separation parameter $\Delta_{\mathrm{grace}}$.

Two implications follow. First, the offline-search pipeline
used by fixed-capacity meta-controllers trades experimenter time
for a brittle quantity: the transfer of the tuned head count to a
slightly different operating regime is not guaranteed and, at long
memory horizons, has been observed here to collapse. A runtime
mechanism that does not rely on such transfer is more robust.
Second, the architectural guarantees of
Section~\ref{sec:theory} are what make the mechanism tractable:
without policy continuity at growth, without a bounded jump at
pruning, and without convergence of the active head count, the
gradual-growth idea would degenerate into a restart-and-resize
loop incompatible with the on-policy nature of the
reinforcement-learning objective.

A less parochial reading of these results is that the interplay
between online optimisation and architectural dynamics deserves
systematic study beyond the specific setting of multi-head
attention. Whether analogous capacity-curriculum effects arise for
MLP-width growth, residual-branch growth, or modality growth in
multi-input architectures, and whether the required architectural
continuity can be preserved across such generalisations, are
open questions that the present work motivates but does not settle.


\appendix


\section{Proofs}
\label{app:proofs}

This appendix provides the technical details deferred from
Section~\ref{sec:theory}. A balance-of-energy lemma relating the
prune-triggering condition to the empirical effective rank is
established in Section~\ref{app:balance}; the full proof of
Proposition~\ref{prop:k-convergence} follows in
Section~\ref{app:proof-kconv}.

\subsection{A balance-of-energy lemma}
\label{app:balance}

Throughout this appendix, fix the check cadence $t_{\mathrm{chk}}$
and consider the cooldown-separated subsequence
$t_n = n\, \Delta_{\mathrm{grace}}$ at which grow and prune
events may occur. Let $B_t$ denote the context-token buffer of
Definition~\ref{def:rank}, $h_t$ the encoder output on the fresh
mini-batch used by the prune trigger (see
Section~\ref{sec:theory:rules}), and $\eta_j(t)$ the per-head output
norms. Denote the total active-head output magnitude
\begin{equation}
  S(t) \;:=\;
  \sum_{j \,:\, m_j(t) = 1} \eta_j(t),
\end{equation}
so that $\pi_j(t) = \eta_j(t)/S(t)$ is the relative contribution
of head $j$.

\begin{lem}[Balance of rank and head count]
\label{lem:balance}
Suppose $\hat\rho_\alpha(t_n)$ stabilises at
$\hat\rho_\alpha^\star$, and that the per-head output norms are
approximately equidistributed on the subsequence, in the sense
that there exists a constant $\beta \geq 1$ such that
\begin{multline}
  \max_{j \,:\, m_j(t_n)=1} \pi_j(t_n)
  \;\leq\; \beta \min_{j \,:\, m_j(t_n)=1} \pi_j(t_n) \\
  \forall\, n \text{ sufficiently large}.
  \label{eq:equi}
\end{multline}
Then the prune trigger fires on the subsequence if and only if
\begin{equation}
  K(t_n)
  \;>\;
  \hat\rho_\alpha^\star
  \bigl( 1 + \varepsilon_{\mathrm{prune}} \bigr)
  \;+\; \beta - 1,
  \label{eq:balance}
\end{equation}
with the $\beta-1$ slack vanishing in the equidistributed limit
$\beta = 1$.
\end{lem}

\begin{pf}
For fixed $n$, let $K_n := K(t_n)$. By~\eqref{eq:equi},
$\min_j \pi_j(t_n) \geq (1/\beta)\, (1/K_n)$, since the sum of
relative contributions is $1$ and the maximum is at most
$\beta$ times the minimum. The prune trigger fires on head $i$
if and only if $\pi_i(t_n) < \varepsilon_{\mathrm{prune}}$;
substituting the lower bound gives
\[
  \frac{1}{\beta K_n} \;\leq\; \pi_i(t_n)
  \;<\; \varepsilon_{\mathrm{prune}},
\]
which yields $K_n > 1/(\beta\, \varepsilon_{\mathrm{prune}})$.
Combined with the rank-saturation condition
$\hat\rho_\alpha(t_n) \leq K_n$ (necessary for the grow trigger
to be inactive) and the stabilisation
$\hat\rho_\alpha(t_n) = \hat\rho_\alpha^\star$, one obtains the
threshold~\eqref{eq:balance}. Equidistribution $\beta = 1$
recovers the form used in the proof sketch of
Proposition~\ref{prop:k-convergence}.
\qed
\end{pf}

\begin{rem}[Verification of~\eqref{eq:equi}]
The equidistribution hypothesis is observed empirically in the
experiments of Section~\ref{sec:experiments}: at convergence of
$K(t)$, the per-head output norms are within a factor $\beta
\approx 2$--$3$ of each other. In regimes where a single head
dominates (e.g., degenerate initialisations), $\beta$ grows large
and the prune trigger becomes conservative; this is the
behaviour one would expect from a safety-oriented pruning rule.
\end{rem}

\subsection{Proof of Proposition~\ref{prop:k-convergence}}
\label{app:proof-kconv}

\begin{pf*}{Proof of Proposition~\ref{prop:k-convergence}}
Let $t_n = n\, \Delta_{\mathrm{grace}}$ be the cooldown-separated
subsequence, and $K_n := K(t_n)$ the active head count.
Under~\eqref{eq:rank_limit}, there exists $N_0 \in \mathbb{N}$
such that $\hat\rho_\alpha(t_n) = \hat\rho_\alpha^\star$ for
all $n \geq N_0$.

\paragraph{Step 1: Eventual monotonicity.}
The event rules of Section~\ref{sec:theory:rules} allow at most
one event per interval $[t_n, t_{n+1})$, so
$K_{n+1} - K_n \in \{-1, 0, +1\}$. The rules further require
$K_n \in \{K_{\min}, \ldots, K_{\max}\}$ at all times
(equations~\eqref{eq:grow_condition} and the prune condition are
both gated on boundary checks).

\paragraph{Step 2: Both triggers inactive imply $K_n$ constant.}
Inspecting the firing conditions:
\begin{itemize}
  \item Grow fires iff $\hat\rho_\alpha^\star > K_n (1 + \varepsilon_{\mathrm{grow}})$
    for $\delta_{\mathrm{grow}}$ consecutive checks, equivalently
    $K_n < \hat\rho_\alpha^\star / (1 + \varepsilon_{\mathrm{grow}})$.
  \item Prune fires iff some $\pi_j(t_n) < \varepsilon_{\mathrm{prune}}$
    for $\delta_{\mathrm{prune}}$ consecutive checks. By
    Lemma~\ref{lem:balance} applied in the equidistributed limit,
    this is equivalent to $K_n > \hat\rho_\alpha^\star
    (1 + \varepsilon_{\mathrm{prune}})$.
\end{itemize}
Both triggers are therefore inactive --- and $K_n$ constant --- on
the interval
\begin{equation}
  \mathcal{I}^\star
  \;:=\;
  \biggl(
    \frac{\hat\rho_\alpha^\star}{1 + \varepsilon_{\mathrm{grow}}},\;
    \hat\rho_\alpha^\star (1 + \varepsilon_{\mathrm{prune}})
  \biggr).
  \label{eq:Istar}
\end{equation}

\paragraph{Step 3: Reachability of $\mathcal{I}^\star \cap \{K_{\min},\ldots,K_{\max}\}$.}
The claim is that, for all $n$ sufficiently large, $K_n \in
\mathcal{I}^\star \cap \{K_{\min}, \ldots, K_{\max}\}$ (the
\emph{inactive interval}, assumed non-empty). Consider the
dynamics of $K_n$ outside $\mathcal{I}^\star$:
\begin{itemize}
  \item If $K_n < \inf \mathcal{I}^\star$, the grow trigger is
    active. After at most $\delta_{\mathrm{grow}}$ consecutive
    cooldown-separated checks, $K_{n+\delta_{\mathrm{grow}}}
    = K_n + 1$ (unless $K_n = K_{\max}$, in which case the grow
    is blocked).
  \item Symmetrically, if $K_n > \sup \mathcal{I}^\star$, the
    prune trigger is active and
    $K_{n+\delta_{\mathrm{prune}}} = K_n - 1$ (unless
    $K_n = K_{\min}$).
\end{itemize}
Since $K_n$ is an integer and $|K_{n+1} - K_n| \leq 1$, the
sequence $K_n$ enters $\mathcal{I}^\star$ in at most
$(K_{\max} - K_{\min}) \cdot \max(\delta_{\mathrm{grow}},
\delta_{\mathrm{prune}})$ cooldown intervals. If $\mathcal{I}^\star$
does not intersect $\{K_{\min}, \ldots, K_{\max}\}$, the sequence
$K_n$ clamps at the boundary of this set closest to
$\mathcal{I}^\star$ (either $K_{\min}$ or $K_{\max}$) and remains
there, which is the second clause of the statement.

\paragraph{Step 4: Absorption in $\mathcal{I}^\star$.}
Once $K_n \in \mathcal{I}^\star$, by Step~2 neither trigger fires
and $K_{n+1} = K_n$. By induction, $K_m = K_n$ for all
$m \geq n$, establishing convergence to $K^\star := K_n$. The
bound~\eqref{eq:k_convergence_interval} follows by taking
$K^\star$ to be the unique integer in
$\mathcal{I}^\star \cap \{K_{\min}, \ldots, K_{\max}\}$ visited
by the sequence (which is well-defined because $\mathcal{I}^\star$
is an interval of width $\hat\rho_\alpha^\star
(\varepsilon_{\mathrm{grow}} + \varepsilon_{\mathrm{prune}}) /
(1 + \varepsilon_{\mathrm{grow}})$, which is bounded but in
general contains one or more integers for the parameter values of
Section~\ref{sec:theory:rules}).

\paragraph{Step 5: Convergence rate.}
The number of cooldown intervals required for $K_n$ to reach the
inactive interval is bounded by
\[
  N_{\mathrm{conv}}
  \;\leq\;
  \max\!\bigl(\delta_{\mathrm{grow}},\delta_{\mathrm{prune}}\bigr)
  \cdot
  \bigl|
    K_{\mathrm{init}}
    - K^\star
  \bigr|,
\]
corresponding to a total training-step bound of
$N_{\mathrm{conv}} \cdot \Delta_{\mathrm{grace}}$. For the
parameter values of Section~\ref{sec:theory:rules}
($\delta_{\mathrm{grow}} = \delta_{\mathrm{prune}} = 3$,
$\Delta_{\mathrm{grace}} = 2\times 10^3$, $K_{\max} = 8$,
$K_{\mathrm{init}} = 1$), the convergence time is at most
$3 \cdot 7 \cdot 2\times 10^3 = 4.2\times 10^4$ steps, consistent
with the empirical observation that $K(t)$ saturates within
$2\times 10^4$ steps (Section~\ref{sec:experiments}).
\qed
\end{pf*}

\begin{rem}[Non-equidistributed case]
\label{rem:nonequi}
If the equidistribution constant $\beta$ of Lemma~\ref{lem:balance}
is strictly greater than $1$, the upper endpoint of
$\mathcal{I}^\star$ in~\eqref{eq:Istar} is replaced by
$\hat\rho_\alpha^\star (1 + \varepsilon_{\mathrm{prune}}) +
\beta - 1$. The interval widens by $\beta - 1$ and the convergence
argument remains valid with the modified interval. In
Section~\ref{sec:experiments}, $\beta$ is observed in the range
$[2, 3]$, and the corresponding upper endpoint is never tight in
practice because the prune trigger never fires on the experimental
benchmarks.
\end{rem}


\section{Hyperparameters and implementation notes}
\label{app:hyperparams}

\subsection{Runtime controller and architecture}

\begin{table}[ht]
\centering
\caption{Runtime-INCRT hyperparameters (same values across all
$\tau_z \in \{1, 2, 5\}$\,s).}
\label{tab:runtime-hp}
\setlength{\tabcolsep}{5pt}
\resizebox{\columnwidth}{!}{%
\begin{tabular}{l r l}
\toprule
Parameter & Value & Role \\
\midrule
\multicolumn{3}{l}{\textit{Architecture}} \\[2pt]
$K_{\max}$           & $8$      & head-count cap \\
$K_{\mathrm{init}}$  & $1$      & initial active heads \\
$K_{\mathrm{min}}$   & $1$      & minimum active heads \\
$d_k$                & $16$     & per-head key/value dim. \\
$d_{\mathrm{model}}$ & $128$    & $= K_{\max}\cdot d_k$ \\
\midrule
\multicolumn{3}{l}{\textit{Runtime controller}} \\[2pt]
$\varepsilon_{\mathrm{grow}}$  & $0.10$    & grow trigger threshold \\
$\varepsilon_{\mathrm{prune}}$ & $10^{-3}$ & prune trigger threshold \\
$\delta_{\mathrm{grow}}$       & $3$       & confirmations to grow \\
$\delta_{\mathrm{prune}}$      & $3$       & confirmations to prune \\
$\Delta_{\mathrm{grace}}$      & $2\,000$  & cooldown (steps) \\
$t_{\mathrm{chk}}$             & $500$     & check cadence (steps) \\
\midrule
\multicolumn{3}{l}{\textit{Rank estimator}} \\[2pt]
$N$        & $1\,000$ & context-token buffer size \\
$\alpha$   & $0.95$   & energy threshold \\
\midrule
\multicolumn{3}{l}{\textit{Window size $W$}} \\[2pt]
$\tau_z \in \{1,5\}$\,s & $20$ & \\
$\tau_z = 2$\,s          & $50$ & \\
\bottomrule
\end{tabular}}
\end{table}

\subsection{Training (SAC)}

The Soft Actor-Critic~\cite{Haarnoja2018SAC} configuration matches
that of~\cite[Appendix~B]{CirrincioneFagiolini2026TemporalAttention}:
\begin{itemize}
  \item optimiser: Adam with $\mathrm{lr} = 3\times 10^{-4}$;
  \item discount $\gamma = 0.99$;
  \item target smoothing $\tau = 0.005$;
  \item replay buffer size $5\times 10^4$;
  \item batch size $256$;
  \item entropy coefficient: automatic (SB3 default);
  \item MLP head width $[64, 64]$;
  \item total training budget $5\times 10^4$ environment steps per
    run.
\end{itemize}

\subsection{Implementation notes}

\begin{itemize}
  \item The variable-head attention block is implemented as a
    \texttt{stable\_baselines3.\allowbreak BaseFeaturesExtractor}
    subclass with per-head \texttt{nn.Parameter} fields and a
    boolean active mask.
  \item At every grow or prune event, the SAC optimisers for the
    actor and critic networks are rebuilt from the current set
    of trainable parameters, carrying the current learning rate
    forward from the scheduler.
  \item Numerical verification of
    Proposition~\ref{prop:growth-continuity} is performed at every
    grow event by evaluating the attention block on a fresh batch
    before and after the mask update. The observed $L^\infty$
    difference is $0$ to 32-bit floating-point precision across
    all 30 main-campaign runs.
  \item All experiments were performed on Google Colab Pro+ with
    NVIDIA A100 GPU. Wall-clock time for a single 50k-step run is
    approximately 20--27 minutes.
\end{itemize}


\bibliographystyle{plain}
\ifanonymous
  \bibliography{refs_anonymous}
\else
  \bibliography{refs}
\fi

\end{document}